%% file: main_tie.tex
\documentclass[journal]{IEEEtranTIE}
\usepackage{graphicx}
\usepackage{cite}
\usepackage{picinpar}
\usepackage{amsmath}
\usepackage{amsthm} 
\usepackage{url}
\usepackage{flushend}
\usepackage[utf8]{inputenc} 
\usepackage{colortbl}
\usepackage{soul}
\usepackage{multirow}
\usepackage{pifont}
\usepackage{color}
\usepackage{alltt}
\usepackage[hidelinks]{hyperref}
\usepackage{enumerate}
\usepackage{siunitx}
\usepackage{breakurl}
\usepackage{epstopdf}
\usepackage{pbox}

\usepackage[english]{babel}

\usepackage{subfig}
\usepackage{booktabs}
\usepackage{tabularx} 
\usepackage{utils/symbols}
\usepackage{tikz}
\usepackage{amssymb}
\usepackage{algorithm}
\usepackage{algorithmicx}
\usepackage{algpseudocode}
\usepackage{mathtools}
\usepackage{caption}
\usepackage{placeins}

\usepackage{soul} 
\usepackage{marginnote} 

\newif\ifreviseenabled
\reviseenabledfalse

\newcommand{\revise}[2]{%
    \ifreviseenabled
        \textcolor{red}{#2}
        \protect\marginnote{\color{red}\small C#1} 
    \else
        #2 
    \fi
}

\usepackage{xcolor}
\usepackage{tcolorbox}
\tcbuselibrary{breakable} 

\ifreviseenabled
    \newtcolorbox{revisebox}[2][]{
        breakable, 
        colback=blue!5!white, 
        colframe=blue!50!black, 
        title={\textbf{Revisions to Comment~#2}}, 
        fonttitle=\small,
        left=0mm,right=0mm,top=1mm,bottom=1mm, 
        #1 
    }
    \newtcolorbox{revisebox_nobreak}[2][]{
        colback=blue!5!white, 
        colframe=blue!50!black, 
        title={\textbf{Revisions to Comment~#2}}, 
        fonttitle=\small,
        left=0mm,right=0mm,top=1mm,bottom=1mm, 
        #1 
    }
\else
    \NewDocumentEnvironment{revisebox}{O{} m}{%
    }{%
    }
    \NewDocumentEnvironment{revisebox_nobreak}{O{} m}{%
    }{%
    }
\fi

\newtheorem{assumption}{Assumption}
\newtheorem{definition}{Definition}
\newtheorem{lemma}{Lemma}
\newtheorem{remark}{Remark}
\newtheorem{theorem}{Theorem}

\newtheorem{problem}{Problem}

\begin{document}
\title{KCFRC: Kinematic Collision-Aware Foothold \\ Reachability Criteria for Legged Locomotion}

\author{Lei Ye, Haibo Gao, Huaiguang Yang, Peng Xu$^{*}$, Haoyu Wang, Tie Liu, Junqi Shan,\\Zongquan Deng and Liang Ding$^{*}$
    \thanks{The authors are with the State Key Laboratory of Robotics and Systems, Harbin Institute of Technology, Harbin 150001, China.}
    \thanks{*Corresponding author: Peng Xu (pengxu\_hit@163.com), Liang Ding (liangding@hit.edu.cn).}
    \thanks{This work was supported by the National Key R\&D Program of China (2022YFB4702300);
        the National Natural Science Foundation of China (Grant No. 91948202);
        the Fundamental Research Funds for the Central Universities (FRFCU9803500621, No. HIT. OCEF. 2023042);
        the Heilongjiang Postdoctoral Fund under Grant LBH-Z20136.}
}

\maketitle

\begin{abstract}

    Legged robots face significant challenges in navigating complex environments, as they require precise real-time decisions for foothold selection and contact planning.
    While existing research has explored methods to select footholds based on terrain geometry or kinematics, a critical gap remains: \revise{2.4}{few existing methods efficiently validate} the existence of a non-collision swing trajectory. This paper addresses this gap by introducing KCFRC, a novel approach for efficient foothold reachability analysis.
    We first formally define the foothold reachability problem and establish a sufficient condition for foothold reachability. Based on this condition, we develop the KCFRC algorithm, which enables robots to validate foothold reachability in real time.
    Our experimental results demonstrate that KCFRC achieves remarkable time efficiency, completing foothold reachability checks for a single leg across 900 potential footholds in an average of 2 ms. Furthermore, we show that KCFRC can accelerate trajectory optimization and is particularly beneficial for contact planning in confined spaces, enhancing the adaptability and robustness of legged robots in challenging environments.

\end{abstract}


\markboth{IEEE TRANSACTIONS ON INDUSTRIAL ELECTRONICS}%
{}

\input{sections/1_introduction}
\input{sections/2_related_work}
\input{sections/3_reachability_prob}

\input{sections/4_kcfrc_check}
\input{sections/5_experiments}

\input{sections/6_applications}

\input{sections/7_conclusion}

\bibliographystyle{Bibliography/IEEEtranTIE}
\bibliography{library}

\vspace{-1.7cm}

\begin{IEEEbiography}[{\includegraphics[width=1in,height=1.25in,clip,keepaspectratio]{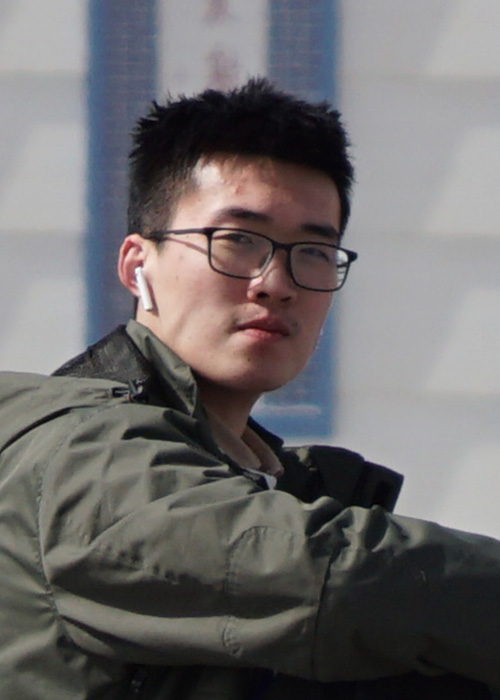}}]
    {Lei Ye} received his Bachelor degree in Automation from Harbin Institute of Technology, Harbin, China, in 2024. He is a Master student in the Department of Aerospace Science and Technology at Harbin Institute of Technology. His research interests include legged robot planning, colony planning, and motion control.
\end{IEEEbiography}
\vspace{-1.7cm}

\begin{IEEEbiography}[{\includegraphics[width=1in,height=1.25in,clip,keepaspectratio]{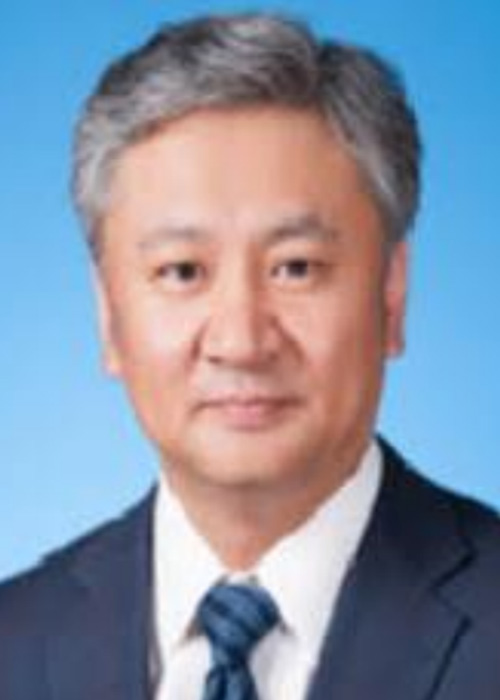}}]
    {Haibo Gao} was born in 1970. He received the Ph.D. degree in mechanical design and theory from the Harbin Institute of Technology, Harbin, China, in 2004. He is currently a Professor with the State Key Laboratory of Robotics and System, Harbin Institute of Technology. His current research interests include specialized and aerospace
    robotics and mechanism.
\end{IEEEbiography}

\begin{IEEEbiography}[{\includegraphics[width=1in,height=1.25in,clip,keepaspectratio]{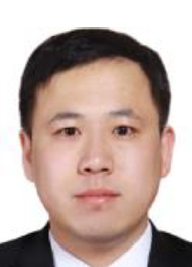}}]
    {Huaiguang Yang} was born in 1991. He received the Ph.D. degree in the manufacturing engineering of aerospace vehicle from the Harbin Institute of Technology, Harbin, China, in 2020.
    He is currently an associate Professor with the State Key
    Laboratory of Robotics and System, Harbin Institute of Technology. His current research interests include terramechanics and dynamic simulation for planetary rovers.
\end{IEEEbiography}
\vspace{-1.6cm}

\begin{IEEEbiography}[{\includegraphics[width=1in,height=1.25in,clip,keepaspectratio]{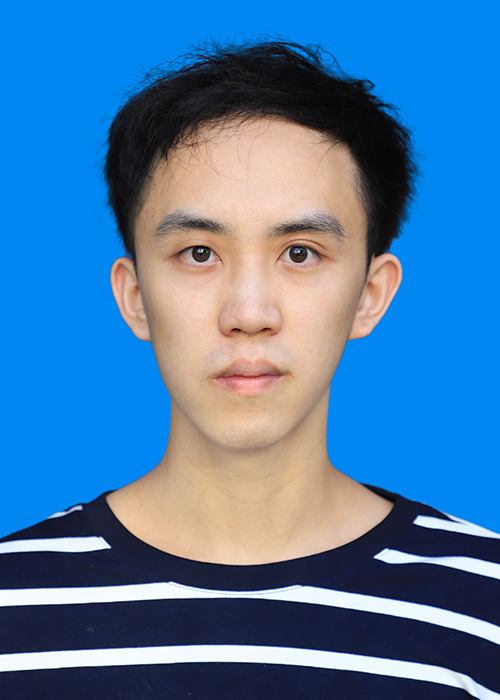}}]
    {Peng Xu} received the Master degree in mechanical
    engineering from the Harbin Institute of Technology, Harbin, China, in
    2020.
    He is currently a Ph.D. student at the State Key Laboratory of Robotics and Systems, Harbin Institute of Technology. His current research interests include motion planning, unsupervised perception, and optimal control of legged robots.
\end{IEEEbiography}
\vspace{-1.6cm}

\begin{IEEEbiography}[{\includegraphics[width=1in,height=1.25in,clip,keepaspectratio]{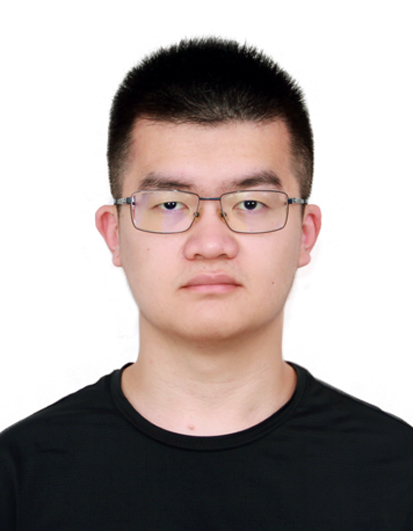}}]
    {Haoyu Wang} received the Bachelor's degree in Automation from Harbin Institute of Technology, Harbin, China, in 2023.
    He is currently pursuing a Master's degree in Mechanical Engineering at Harbin Institute of Technology, Harbin, China. His current research interests include optimal control of legged robots.
\end{IEEEbiography}
\vspace{-1.6cm}

\begin{IEEEbiography}[{\includegraphics[width=1in,height=1.25in,clip,keepaspectratio]{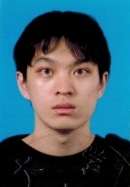}}]
    {Tie Liu} received the master's degree in mechanical
    engineering in 2023 from the Harbin
    Institute of Technology, Harbin, China, where he
    is currently working toward the Ph.D. degree in
    the manufacturing engineering of aerospace vehicle
    with the State Key Laboratory of Robotics
    and Systems.
    His research focuses on the design, modeling, and simulation of Hexapod robots.
\end{IEEEbiography}
\vspace{-1.6cm}

\begin{IEEEbiography}[{\includegraphics[width=1in,height=1.25in,clip,keepaspectratio]{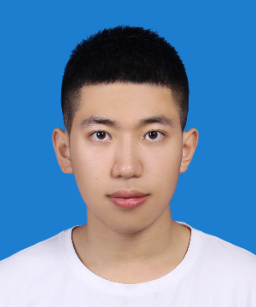}}]
    {Junqi Shan} is currently pursuing his Master's degree at the State Key Laboratory of Robotics and Systems (Harbin Institute of Technology). His research focuses on intelligent path planning, optimization algorithms, and multimodal learning for robotic systems. For detailed research updates and project demonstrations, please visit his academic portfolio at \url{https://tipriest.github.io}.
\end{IEEEbiography}
\vspace{-1.6cm}

\begin{IEEEbiography}[{\includegraphics[width=1in,height=1.25in,clip,keepaspectratio]{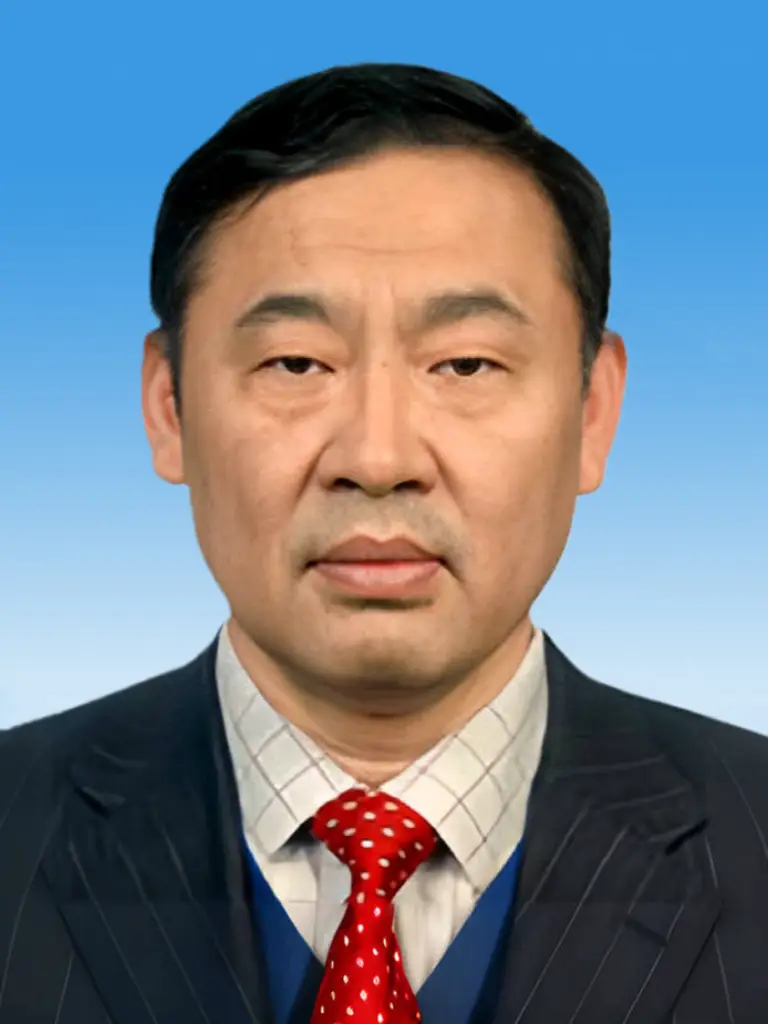}}]
    {Zongquan Deng} was born in 1956. He received the master's degree from the Harbin Institute of Technology, Harbin, China, in 1984. He is currently a Professor and an Academician of the Chinese Academy of Engineering in the Harbin Institute of Technology. His current research interests include special robot systems, and aerospace mechanisms and control.
\end{IEEEbiography}
\vspace{-1.6cm}

\begin{IEEEbiography}[{\includegraphics[width=1in,height=1.25in,clip,keepaspectratio]{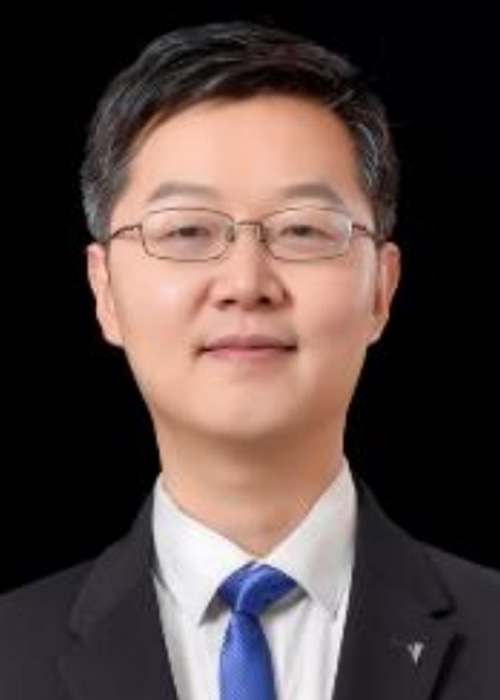}}]
    {Liang Ding}
    (Senior Member, IEEE) was born in 1980. He received the Ph.D. degree in mechanical engineering from the Harbin Institute of Technology. He is currently a Professor with the State Key Laboratory of Robotics and System, Harbin Institute of Technology. His current research interests include terramechanics, intelligent control, design and simulation of robotic systems, in particular for planetary exploration rovers and multi-legged robots.
\end{IEEEbiography}
\vspace{-1.6cm}

\end{document}

%% file: sections/1_introduction.tex
\section{INTRODUCTION}

\IEEEPARstart{L}{egged} robots have emerged as pivotal mobile platforms for executing complex tasks in unstructured environments. Navigating confined spaces necessitates meticulous motion planning and foothold selection to prevent limb collisions during leg swings. This poses a fundamental challenge: how to guarantee kinematically feasible, collision-free transitions between consecutively planned footholds, particularly when joint space limitations and dynamic body movements constrain the swing leg's operational space. Formally, we define the \textbf{Foothold Reachability Problem (FRP)} as verifying the existence of a collision-free trajectory that connects two adjacent footholds under prescribed torso motion, considering both environmental obstacles and the robot's kinematic constraints.

This work focuses on efficient \textbf{Foothold Reachability Criteria (FRC)} considering leg kinematic constraints and limb-terrain collision for legged locomotion,
which plays a crucial role in legged robot contact planning. This is especially important for some \revise{2.5}{non-mammalian} configuration robots with joints exposed outside the base-ground projection\cite{2019Buchanan_Weaver} that easily collide with the environment.

\begin{figure}
    \centering
    \includegraphics[width=1\linewidth]{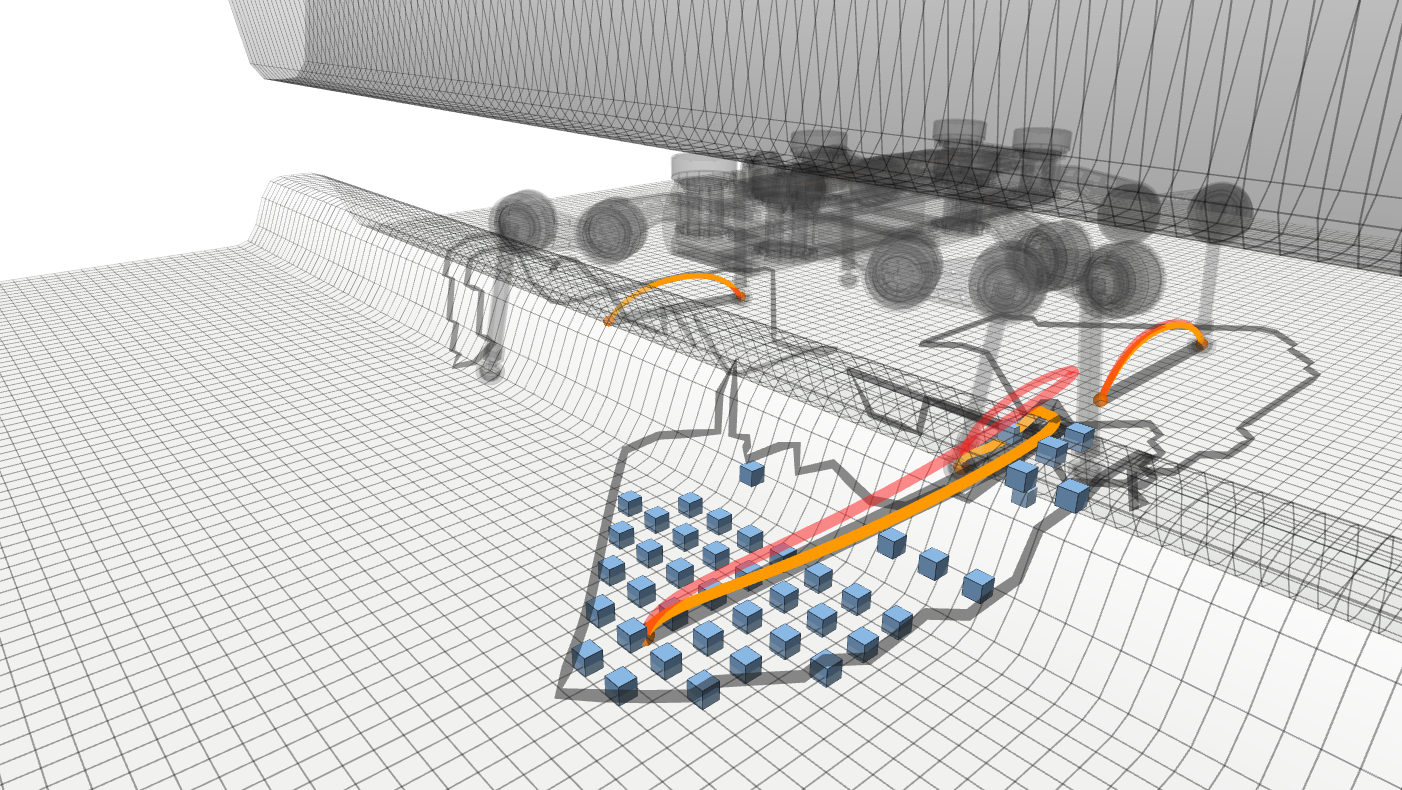}
    \caption{Kinematic Collision-Aware Foothold Reachability Criteria
        \revise{2.6}{demonstration in barrier traversal. Blue cubes show validated reachable footholds, orange/red curves show initialized/optimized swing trajectories. The hexapod robot must carefully select footholds to navigate through the confined space.}
        \revise{3.9}{More supplemental materials are available at
            \href{https://masteryip.github.io/fltplanner.github.io/}{\color{blue}{KCFRC Webpage}}}.}
    \label{fig:1_reachability_example}
\end{figure}

As illustrated in Fig. \ref{fig:1_reachability_example}, successful barrier traversal by the hexapod requires simultaneous satisfaction of multiple constraints: avoiding foot-ground collisions, preventing knee-ceiling interference, and maintaining kinematic feasibility throughout the motion.

\begin{revisebox}{1.1,4.1}
    While the FRP shares similarities with classical motion planning problems, existing approaches have limitations for real-time legged locomotion applications.
    Classical sampling-based planners such as PRM and RRT* can provide probabilistically complete solutions for reachability queries, but typically require substantial computation time (tens to hundreds of milliseconds) that exceeds real-time planning requirements.
    Recent advances in convex-set methods for manipulator motion planning, particularly Graph of Convex Sets (GCS)\cite{2023Marcucci_GCS}, offer elegant solutions but require complex geometric reasoning\cite{2023Petersen_GCS_gen} and struggle with the dynamic nature of moving-base scenarios in legged locomotion.
    GPU-accelerated batch inverse kinematics approaches, such as cuRobo\cite{2023Balakumar_curobo_report} and similar frameworks, can achieve fast collision checking during trajectory sampling but require specialized hardware, significant memory overhead for collision representation maintenance, and typically operate as post-hoc filters rather than providing a priori reachability guarantees.

    For legged robot specific approaches, Fahmi et al. proposed Foothold Evaluation Criteria (FEC) for reachability checks on predefined swing trajectories\cite{2023fahmi_ViTAL}, as shown in Fig. \ref{fig:1_FEC}. This method considers foot collision, kinematic feasibility, and leg collision while being limited to fixed trajectory patterns.
    Implicit contact planning methods such as MPC/WBC-based controllers\cite{2018fankhauser, 2022grandia_pnmpc} and RL controllers\cite{2024Miki_walk_in_confined_space} formulate the planning-control problem with related constraints and heuristic reference trajectories but do not provide explicit reachability guarantees before trajectory optimization. Without efficient reachability pre-screening, these methods may struggle with corner cases like Fig. \ref{fig:1_reachability_example} or waste computational resources on infeasible problems.
\end{revisebox}

\begin{figure}
    \centering
    \includegraphics[width=1.0\linewidth]{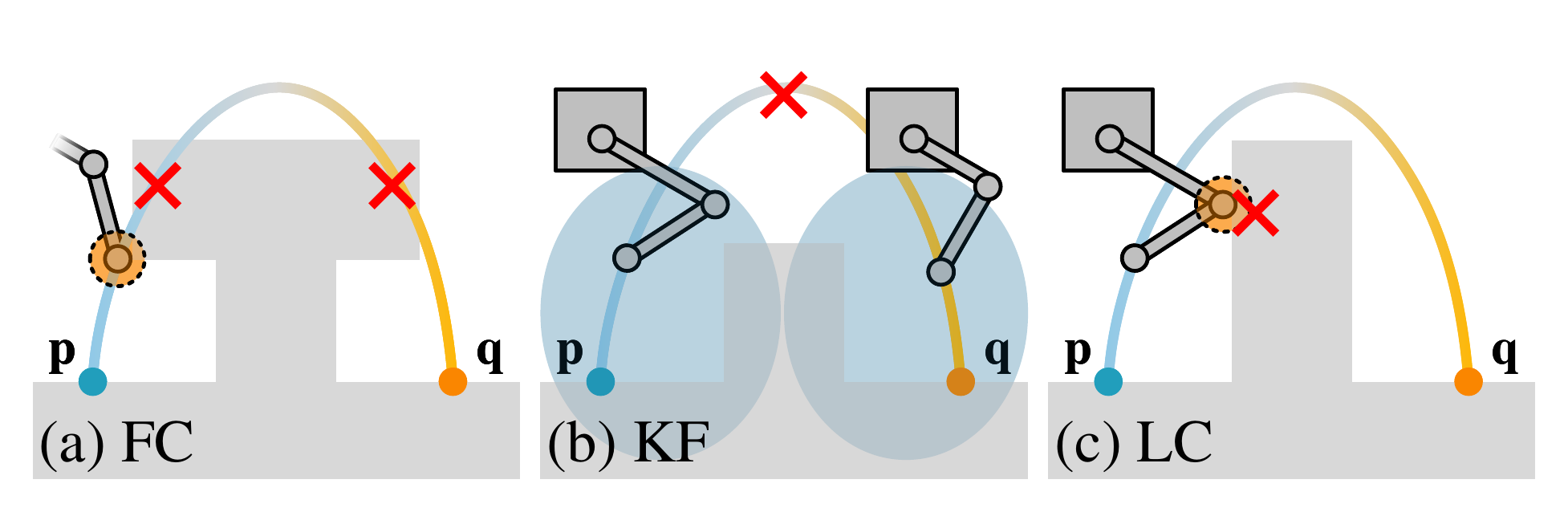}
    \caption{Foothold Evaluation Criteria.
        \revise{2.6}{(a) Foot Collision checking for foot-terrain interference, (b) Kinematic Feasibility ensuring joint limits compliance, and (c) Leg Collision detecting segment-obstacle interference. These criteria require explicit trajectory generation for validation.}}
    \label{fig:1_FEC}
\end{figure}

This paper proposes \textbf{K}inematic \textbf{C}ollision-Aware \textbf{F}oothold \textbf{R}eachability \textbf{C}riteria for legged robot contact planning, called \textbf{KCFRC}. \revise{1.1,4.1}{Our method provides a computationally efficient approach to foothold reachability analysis that complements existing trajectory planning methods. It offers a unique combination of theoretical guarantees and computational efficiency specifically designed for real-time legged locomotion applications. We define the FRP for legged robots mathematically and demonstrate a sufficient condition for swing trajectory existence using topological analysis. The method is implemented and validated on the ElSpider 4 Air robot.}

\begin{revisebox}{1.1,4.1,4.6}
    To the best of our knowledge, our approach is the first to provide theoretical guarantees about foothold reachability while serving as a fast preprocessing filter that can reduce computational burden for subsequent trajectory planning.
    The method achieves significant speed improvements (100-400x faster than sampling-based methods) while maintaining high accuracy for the specific case of 3-DOF leg chains. The key contributions of this work include:

    \begin{enumerate}
        \item \textbf{A theoretical framework providing sufficient conditions for foothold reachability.} We formulate the FRP mathematically and demonstrate a topologically-based sufficient condition that enables rapid reachability assessment without full trajectory optimization.
        \item \textbf{An efficient algorithmic implementation with substantial computational improvements.} We combine the proposed theoretical sufficient conditions with practical collision and kinematic constraint checking to develop a fast foothold reachability assessment method. The open-source implementation\footnote{\url{https://github.com/MasterYip/FLTPlanner}} achieves 100-400x speed improvements over existing methods while maintaining high accuracy.
        \item \textbf{Comprehensive validation across multiple robot platforms and environments.} The method is validated through extensive simulation and hardware experiments on the ElSpider 4 Air hexapod and Unitree A1 quadruped robots, demonstrating its generalizability across different morphologies and its effectiveness in real-world confined space navigation scenarios.
    \end{enumerate}
\end{revisebox}

%% file: sections/2_related_work.tex
\section{RELATED WORK}

\subsection{Reachability in Robotics}

Reachability in robotics varies across applications. In navigation, traversability maps \cite{2013Papadakis_Survey_Traversability,2022Borges_Survey_on_Traversability} define robot-accessible regions using terrain and motion constraints. For manipulators, reachability maps characterize end-effector pose accessibility \cite{2007Zacharias_Workspace_reachability}, with recent advances using convex set graphs for planning efficiency \cite{2023Marcucci_GCS} and dynamic grasping \cite{2021Akinola_Dynamic_Grasp_with_Reachability}.

In legged locomotion, sequential reachability constraints for torso poses \cite{2021Buchanan,2022Xu_MCTS,2018Tonneau_Acyclic_Contact_Planner} often neglect detailed foothold analysis. Implicit planning methods like MPC \cite{2022grandia_pnmpc} and RL \cite{2024Miki_walk_in_confined_space} risk failure in confined spaces by overlooking foothold/trajectory reachability, highlighting the need for integrated reachability validation.

\begin{revisebox}{2.3}
    \subsection{Foothold Reachability from Terrain Perspective}
    Early foothold selection methods focused primarily on terrain morphology. Kolter et al. \cite{2008Kolter_LittleDog} introduced terrain-aware control using foot-cost maps derived from height map features, while Kalakrishnan et al. \cite{2010Kalakrishnan} improved this through expert demonstration-based foothold ranking. Similar terrain-based approaches analyzing inclination, roughness, and geometry have been validated across various platforms \cite{2010Belter,2017Li, 2019Griffin_Footstep_humanoid}.
\end{revisebox}

\subsection{Foothold Reachability from Kinematics Perspective}

\revise{2.3}{Nevertheless, footholds should not be determined solely by terrain without considering robot kinematics.}
Fankhauser et al. \cite{2018fankhauser} and Jenelten et al. \cite{2020Jenelten_online_foothold_optimization} developed ANYmal locomotion planners incorporating foothold-score maps with kinematic feasibility constraints, while Grandia et al.\cite{2022grandia_pnmpc} used polygon dilation to reduce computational demand.
Although these methods reflect foothold reachability through optimization constraints, they do not explicitly define reachability or ensure the existence of non-collision swing trajectories.

Foothold Evaluation Criteria (FEC) \cite{2023fahmi_ViTAL} assess foot collision, kinematic feasibility, and leg collision along predefined swing trajectories (Fig. \ref{fig:1_FEC}). CNN acceleration methods \cite{2019Villarreal_Foothold_Adaptation_CNN} further optimize FEC computation.
However, this approach fails to rigorously assess the existence of feasible swing trajectories, relying instead on validity checks confined to predefined candidate trajectories or simplistic heuristics such as maximum step height thresholds.

\begin{revisebox}{1.1,4.1}
    \subsection{Foothold Reachability Check through Trajectory Planner}
    Trajectory planners have been applied for swing trajectory planning in works \cite{2011Zucker_chomp_littledog, 2018fankhauser} after footholds are decided.
    Intuitively, they can solve foothold reachability by searching for feasible swing trajectories between adjacent footholds given torso trajectory.
    However, checking foothold reachability in batches using trajectory planners is inefficient.
    Although sampling-based planners like PRM\cite{1996Kavraki_PRM} and RRT\cite{1998LaValle_RRT} are probabilistically complete, they typically require exhaustive sampling and substantial computation time in solving FRP that exceeds real-time requirements.
    Gradient-based planners\cite{2013zucker_CHOMP, 2014Schulman_TrajOpt} suffer from local minima and require initial trajectories.
    Recent convex-set methods like Graph of Convex Sets (GCS)\cite{2023Marcucci_GCS} offer elegant solutions but require complex geometric reasoning\cite{2023Petersen_GCS_gen} and struggle with moving-base scenarios in legged locomotion.
    GPU-accelerated approaches like cuRobo\cite{2023Balakumar_curobo_report} achieve fast collision checking but require specialized hardware and operate as post-hoc filters rather than providing a priori reachability guarantees.
\end{revisebox}

%% file: sections/3_reachability_prob.tex
\section{Overview of the Work}

The paper's structure follows the framework illustrated in Fig. \ref{fig:reachability_framework}. We first formulate the FRP in Section \ref{section:frp_formulation}, followed by deriving a sufficient foothold reachability condition in Section \ref{section:sufficient_cond}. Section \ref{section:KCFRC} details the KCFRC implementation through four sequential stages (Fig. \ref{fig:reachability_framework}b-e), supplemented by an explanatory animation\footnote{\url{https://masteryip.github.io/fltplanner.github.io/static/videos/method_standalone.mp4}}. Section \ref{section:experiment} analyzes it's performance through quantitative benchmarks, while Section \ref{section:applications} demonstrates applications KCFRC. The work concludes in Section \ref{section:conclusion}.

\section{Foothold Reachability Problem}

\begin{table}
    \centering
    \begin{revisebox_nobreak}{2.7, 2.8, 3.6}
        \caption{Symbol \& Abbreviation}
        \label{tab:example}
        \begin{tabular}{cl} 
            \toprule
            \textbf{Symbols}                       & \textbf{Description}                     \\
            \midrule
            $p$                                    & Current foothold                         \\
            $q$                                    & Candidate foothold                       \\
            $X$                                    & Space used in the FRP                    \\
            $D_f$                                  & Feasible domain in the FRP               \\
            $D_o$                                  & Occupied domain in the FRP               \\
            $f_P: \mathbb{E}^3 \to \mathrm{SE}(3)$ & Torso pose query function                \\
            $f_G: \mathbb{E}^2 \to \mathbb{R}$     & Ground elevation function                \\
            $f_C: \mathbb{E}^2 \to \mathbb{R}$     & Ceiling elevation function               \\
            $S_g$                                  & Guiding surface                          \\
            $f_{S_g}: \mathbb{E}^2 \to \mathbb{R}$ & Guiding surface elevation function       \\
            $S_a$                                  & Auxiliary surface                        \\
            $f_{S_a}: \mathbb{E}^2 \to \mathbb{R}$ & Auxiliary surface elevation function     \\
            $B$                                    & 2D Intersection Border of $S_a \cap D_f$ \\
            $i$                                    & Leg index                                \\
            \midrule
            \textbf{Abbreviations}                 & \textbf{Description}                     \\
            \midrule
            KCFRC                                  & Kinematic Collision-Aware                \\
                                                   & Foothold Reachability Criteria           \\
            FRP                                    & Foothold Reachability Problem            \\
            FEC                                    & Foothold Evaluation Criteria             \\
            TE                                     & Trajectory Existence                     \\
            KF                                     & Kinematic Feasibility                    \\
            NFC                                    & No Foot Collision                        \\
            NLC                                    & No Leg Collision                         \\
            \bottomrule
        \end{tabular}
    \end{revisebox_nobreak}
\end{table}

\subsection{Problem Formulation and Simplification Assumption}
\label{section:frp_formulation}

\begin{figure}[h]
    \centering
    \includegraphics[width=1\linewidth]{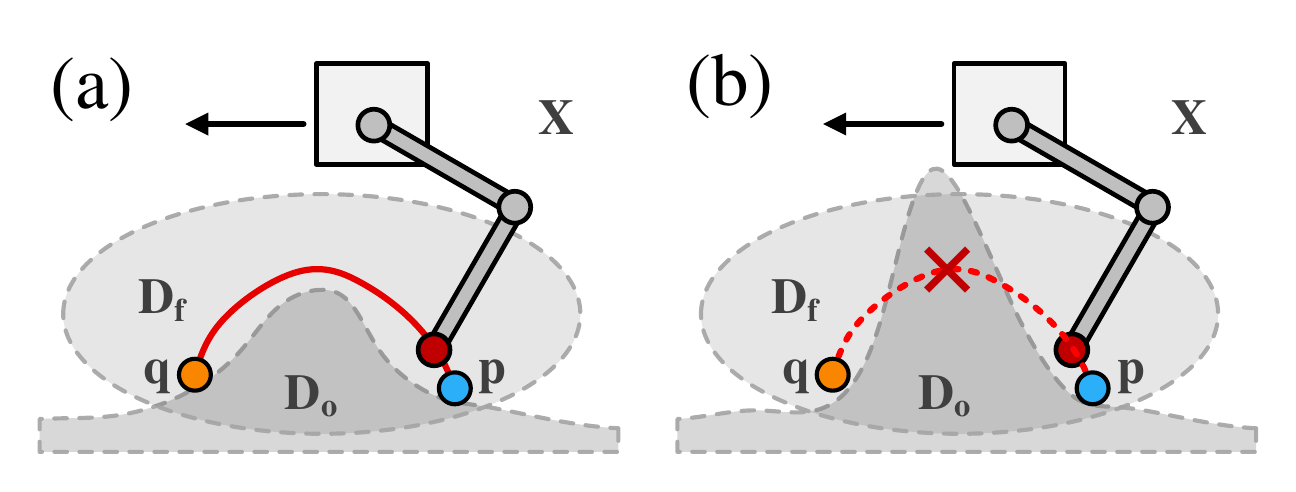}
    \caption{Reachability of the expected foothold. The foothold is reachable if feasible domain is not split by obstacles, and vice versa.}
    \label{fig:reachability_explain}
\end{figure}

\begin{figure*}[t!]
    \centering
    \includegraphics[width=1.0\textwidth]{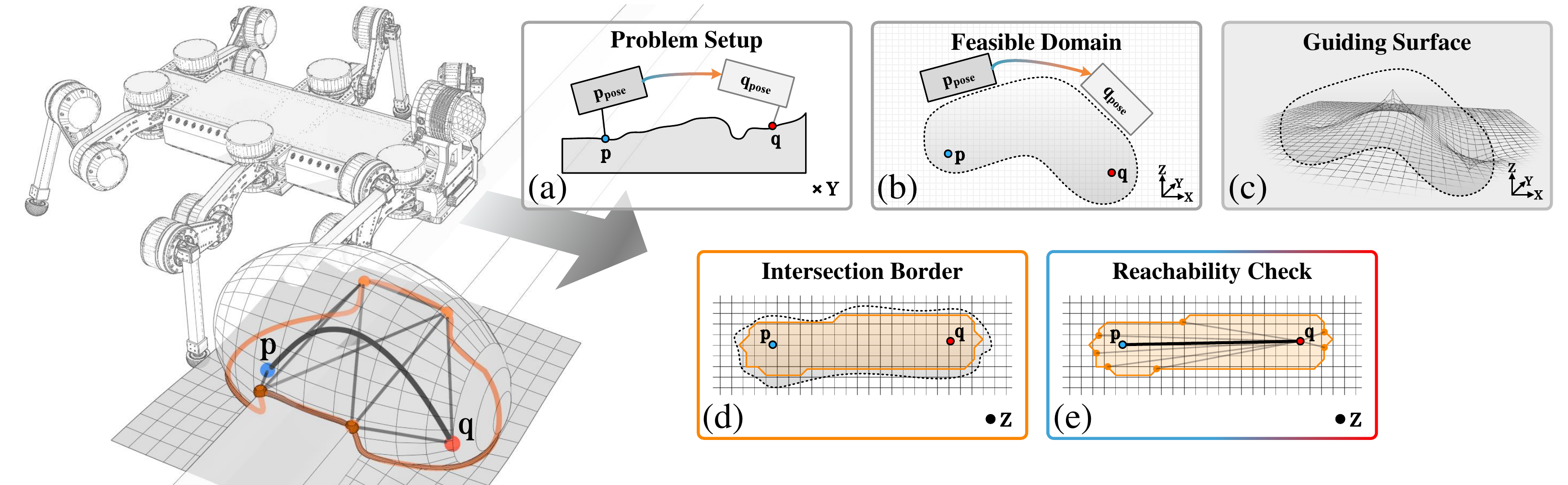}
    \caption{Overview of KCFRC. (a) Problem setup. (b) Feasible domain definition [Section \ref{section:feasible_domain}]. (c) Guiding surface generation [Section \ref{section:guid_surf}]. (d) Intersection border computing [Section \ref{section:intersect_border}]. (e) Reachability check [Section \ref{section:reachability_check}].}
    \label{fig:reachability_framework}
\end{figure*}


\begin{definition}[Foothold Reachability]
    As is shown in Fig. \ref{fig:reachability_explain}, we define the reachability of a foothold as below.
    Given torso trajectory and robot specifications, the next foothold is reachable only if the following conditions are met:

    \begin{enumerate}
        \item Trajectory Existence (TE): There exists a swing trajectory connecting lift-off and touch-down footholds during base movement.
        \item Kinematic Feasibility (KF): Joint limits are satisfied at each point of the trajectory.
        \item No Foot Collision (NFC): The foot does not collide with the terrain during the swing phase.
        \item No Leg Collision (NLC): Any part of the leg does not collide with the terrain during the swing phase.
    \end{enumerate}
\end{definition}

To clarify this problem, we formulate the reachability problem mathematically as follows.



\begin{problem}[Reachability Problem]
\label{prob:reachability}
Let \( D_f, D_o \subseteq X \) be non-empty subsets of a topological space \( X \), where \( D_f \) is a connected closed set and \( D_o \) is an open set. Given points \( p, q \) such that \( p, q \in D_f \), \( p, q \notin D_o \), and \( p \neq q \). Determine an efficient method to decide whether there exists a path \( f: [0, 1] \to (D_f \setminus D_o) \) such that \( p = f(0) \) and \( q = f(1) \).
\end{problem}


This problem is the mathematical abstraction of FRP, the following assumptions are proposed to make this problem concrete.

\begin{assumption}[Assumption for FRP]
    \label{assumption:specific_problem}
    For the consideration of computational feasibility, this project has the following specific requirements for the problem:

    (1) $X$ is the three-dimensional Euclidean space $\mathbb{E}^3$.

    (2) $D_f$ is a three-dimensional connected feasible domain.
    \revise{1.2,4.3}{
        We assume the robot torso trajectory is known during a single step, and
        there exists a continuous map $f_{P}: \mathbb{E}^3 \to \mathrm{SE}(3)$
        such that for every point $s$ in $D_f$ exists a corresponding robot base pose $f_{P}(s)$
        satisfies Kinematic Feasibility (KF) and No Leg Collision (NLC) conditions.}
    The specific definition of $D_f$ and leg collision checking method will be introduced later.

    (3) $D_o$ is the occupied domain of the three-dimensional terrain, including ceiling and ground. It can be represented as
    $$D_o : \{ (x, y, z) \in X \mid z < f_{G}(x, y), z > f_{C}(x, y) \}$$
    where $f_{G}(x, y)$ is the height map of the terrain, and $f_{C}(x, y)$ is the height map of the ceiling. Both are single-valued continuous functions, and $f_{G}(x, y) < f_{C}(x, y), \forall (x, y) \in \mathbb{E}^2$.

    (4) $p, q$ are the adjacent footholds of a leg in one step. They satisfy $p, q \in \{ (x, y, z) \in X \mid z = f_{G}(x, y) \}$.
\end{assumption}

\begin{remark}
    Under this assumption, $D_f$ is a time-invariant feasible domain. However, the robot is moving during locomotion, hence some approximations are adopted when defining $D_f$.
\end{remark}

\subsection{Sufficient Condition for Foothold Reachability}
\label{section:sufficient_cond}

\begin{lemma}
    \label{lemma:reachability}
    For Problem \ref{prob:reachability},
    if \( D_f \setminus D_o \) is homeomorphic to \( D_f \),
    then there exists a path \( f: [0, 1] \to (D_f \setminus D_o) \) such that \( p = f(0) \) and \( q = f(1) \).
\end{lemma}

Intuitively, Lemma \ref{lemma:reachability} is a sufficient condition for foothold reachability. Nevertheless, it is not practical for algorithm implementation. Therefore, we introduce the concept of guiding surface, which acts as a reference and helps in dimension reduction.

\begin{definition}[Guiding Surface $S_g$]
    \label{def:guiding_surface}
    For problem \ref{prob:reachability}, under the conditions of Assumption \ref{assumption:specific_problem},
    if a single-connected two-dimensional surface $S_g$ satisfies the following conditions:

    (1) $S_g : \{(x,y,z) \in X \mid z = f_{S_g}(x, y)\}$, where $f_{S_g}$ is a single-valued continuous function.

    (2) $p, q \in S_g$.

    then $S_g$ is called a guiding surface.
\end{definition}

\begin{lemma}[Existence of Guiding Surface]
    \label{lemma:guiding_surface_existence}
    For problem \ref{prob:reachability}, under the conditions of Assumption \ref{assumption:specific_problem}, a guiding surface $S_g$ must exist.
\end{lemma}

\begin{proof}
    Since $p, q \in \{ (x, y, z) \in X \mid z = f_{G}(x, y) \}$ and $f_{G}(x, y)$ is a single-valued continuous function,
    we can take $S_g : \{(x, y, z) \in X \mid z = f_{G}(x, y)\}$ as the guiding surface.
    Proof completed.
\end{proof}

\begin{remark}
    Although the existence of the guiding surface $S_g$ is evident,
    the significance of the guiding surface lies in providing a default reachable trajectory without considering terrain influences.
    Therefore, designing the guiding surface $S_g$ should ideally satisfy that $S_g \cap D_f$ is single-connected.
    However, a guiding surface that meets this condition may not always exist.
\end{remark}

After that, we define the auxiliary surface $S_a$ to ensure that $S_a$ does not intersect with the ceiling and ground.

\begin{definition}[Auxiliary Surface $S_a$]
    \label{def:aux_surface}
    For problem \ref{prob:reachability},
    under the conditions of Assumption \ref{assumption:specific_problem},
    if there exists a guiding surface $S_g$, then an auxiliary surface $S_a : \{(x, y, z) \in X \mid z = f_{S_a}(x, y)\}$ can be defined as follows:
    \begin{equation}
        \label{eq:aux_surf}
        f_{S_a}(x, y) = \begin{cases}
            f_G(x, y),     & \text{if } f_G(x, y) > f_{S_g}(x, y) \\
            f_C(x, y),     & \text{if } f_C(x, y) < f_{S_g}(x, y) \\
            f_{S_g}(x, y), & \text{otherwise}
        \end{cases}
    \end{equation}
\end{definition}

\begin{theorem}[Sufficient Condition for Reachability]
    \label{theorem:reachability2}
    For problem \ref{prob:reachability}, under the conditions of Assumption \ref{assumption:specific_problem},
    if the auxiliary surface $S_a$ corresponding to the guiding surface $S_g$ satisfies that $S_a \cap D_f$ is connected,
    then there exists a path $f: [0, 1] \to (D_f - D_o)$, satisfying $p = f(0), q = f(1)$.
\end{theorem}

\begin{figure}[h]
    \centering
    \includegraphics[width=1\linewidth]{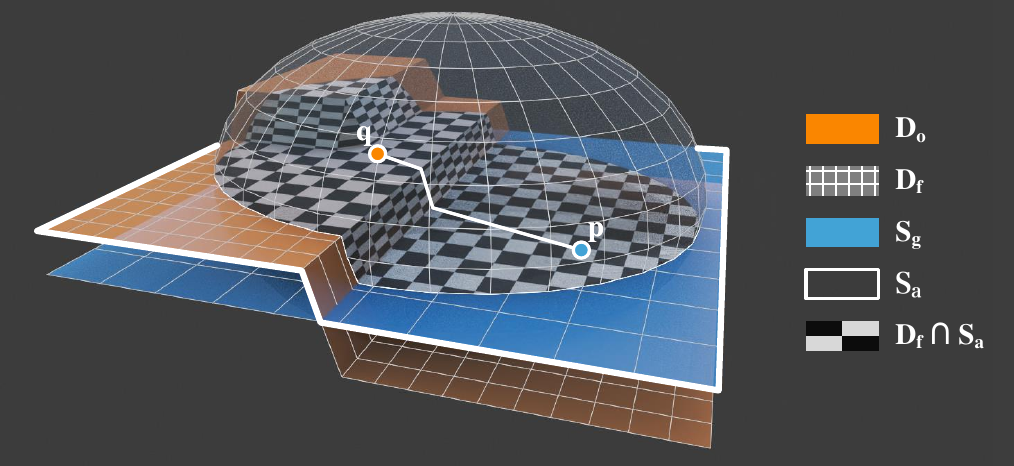}
    \caption{Diagram of the Sufficient Condition for Reachability (Theorem \ref{theorem:reachability2}).}
\end{figure}

\begin{proof}
    Since $D_o : \{ (x, y, z) \in X \mid z < f_{G}(x, y), z > f_{C}(x, y) \}$,
    we have $S_a \cap D_o = \emptyset$, so $S_a \cap (D_f - D_o) = S_a \cap D_f$.
    Given that $S_a \cap D_f$ is connected, it follows that $S_a \cap (D_f - D_o)$ is connected, hence there exists a path $f: [0, 1] \to (D_f - D_o)$, satisfying $p = f(0), q = f(1)$.
\end{proof}


As long as $D_f, D_o, S_g, p, q$ are decided, it is possible to check the foothold reachability according to Theorem \ref{theorem:reachability2}. The establishment of $D_f, D_o, S_g$ will be discussed in the following sections.
It is worth noting that the condition in Theorem \ref{theorem:reachability2} is a sufficient condition for the existence of solutions to the non-collision swing trajectory.
However, a well-defined guiding surface may adapt to most cases in legged robot locomotion.

%% file: sections/4_kcfrc_check.tex
\section{Kinematic Collision-Aware FRC}
\label{section:KCFRC}


\subsection{Terrain Representation}

We represent the environment and obstacles in the form of Grid Map\cite{2016Fankhauser_GridMap}, which is widely used in legged locomotion. To represent obstacles in confined spaces, the ground and ceiling layers are both defined and generated in real-time.

\subsection{Kinematics and Collision Detection}
\label{section:coll_detect}

\begin{figure}[h]
    \centering
    \includegraphics[width=0.9\linewidth]{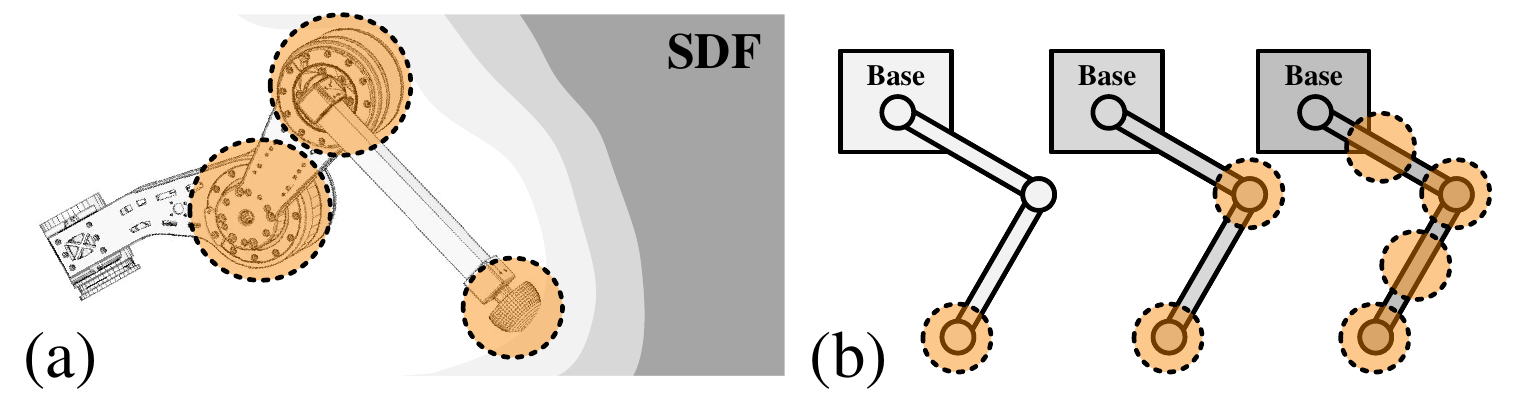}
    \caption{Collision detection method. (a) Collision detection under SDF. (b) Foot-collision-only, joint-collision-only and link-subdivision collision models.}
    \label{fig:collision_ball}
\end{figure}

Leg kinematics and inverse kinematics are computed using Pinocchio\cite{2019carpentier_pinocchio} and IKFast\cite{2010Diankov_IKFast}, respectively. For collision detection, we compute the Signed Distance Field (SDF) from the grid map \revise{4.7}{asynchronously in a separate thread} and model collision volumes using sets of spheres (Fig. \ref{fig:collision_ball}). This sphere-based approach enables efficient collision validation through distance queries against the precomputed SDF.

\subsection{Feasible Domain}
\label{section:feasible_domain}

For a 3-DOF robot leg, the conversion from configuration space $C$ to task space $T$ is a surjective mapping that may become bijective with appropriate joint limits. The space-time feasible domain $D_{fi}$ represents the region swept by task space $T_i$ during single-step movement:
\begin{equation}
    D_{fi} : \{(x,y,z,t) \mid (x,y,z) \in T_i(t), t_0 < t < t_1\}
\end{equation}

\begin{figure}[h]
    \centering
    \includegraphics[width=0.8\linewidth]{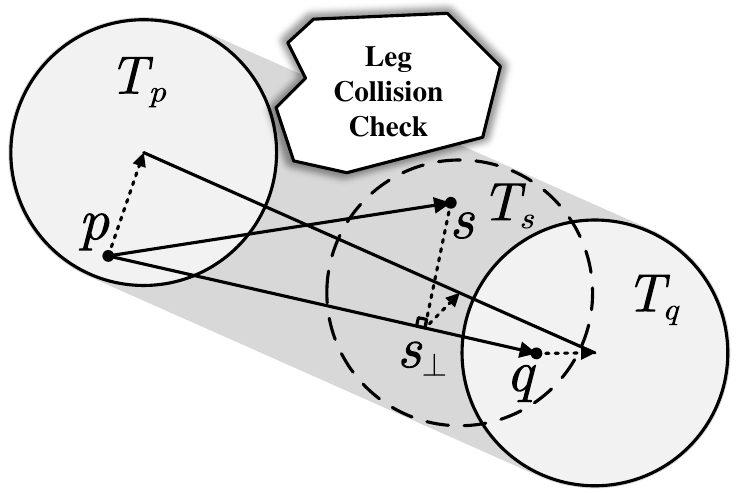}
    \caption{Approximate feasible domain PITD.}
    \label{fig:PITD}
\end{figure}

Since $D_{f}$ is defined in the space-time domain which causes troubles in calculation, we introduce the \textbf{Pose Interpolation Task Domain (PITD)} approximation. Given base poses $P_p, P_q \in \mathrm{SE}(3)$ and leg task space $T_{ib}$ under base coordinates, we compute interpolated task spaces as shown in Fig. \ref{fig:PITD}.

\begin{revisebox}{1.2,4.3}
    The PITD formulation addresses foothold reachability by approximating the continuous leg workspace evolution during robot base movement. As the robot transitions from pose $P_p$ to $P_q$, the interpolation parameter $r(s)= l_{ps_\perp} / l_{pq} \in [0,1]$ represents the progression along this transition:
    
    $$
        f_P^{PITD}(r) = \exp \left( r \cdot \log( P_qP_p^{-1}) \right) \cdot P_p
    $$

    $$
        T_{s}(r)= f_P^{PITD}(r) \cdot T_{ib}
    $$
    
    Here, $f_P^{PITD}(r)$ computes the interpolated robot base pose using SE(3) exponential mapping, while $T_s(r)$ gives the leg's reachable workspace at the interpolated pose. This enables efficient evaluation of whether candidate footholds remain kinematically accessible throughout the robot's base motion.
\end{revisebox}

The PITD feasible domain for leg $i$ is then defined as:
\begin{equation}
    \label{eq:pitd}
    D_f^{PITD} = \{s \mid s \in T_{s}(r(s)), \mathrm{LegColl}(s) \ne True \}
\end{equation}
where IKFast verifies kinematic feasibility and precomputed gridmap SDF ensures collision-free conditions according to Section \ref{section:coll_detect}.

\subsection{Guiding Surface Design}
\begin{figure}[ht]
    \centering
    \includegraphics[width=0.8\linewidth]{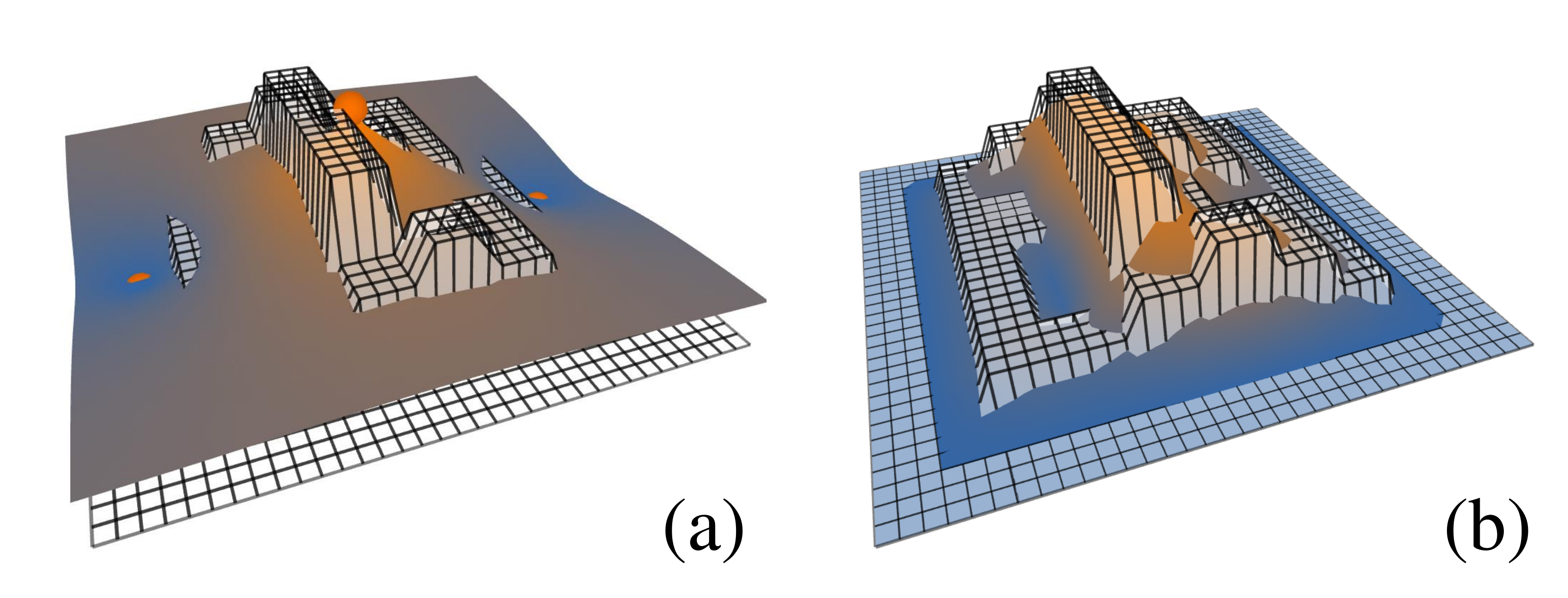}
    \caption{Grid map and generated guiding surfaces. (a) Keypoint-weighted guiding surface \revise{3.4}{creates smooth interpolation between key points with $k=1, w_i=1$}. (b) Convolutional guiding surface \revise{3.4}{smooths terrain features using $n=5, l=0.03$ m kernel}. The grid map resolution is 0.05 m.}
    \label{fig:harmonic_guiding_surface}
\end{figure}
\label{section:guid_surf}

\begin{revisebox}{3.4}
    The guiding surface $S_g$ plays a crucial role in creating the auxiliary surface $S_a$ that helps determine foothold reachability. Intuitively, $S_g$ acts as a "bridge" that guides the leg motion between start and goal positions, ensuring that the intersection $S_a \cap D_f$ forms a connected domain for efficient reachability analysis. The surface should smoothly connect the start and goal footholds while avoiding major obstacles, creating a feasible pathway for leg swing motion.

    \textbf{Keypoint-Weighted Guiding Surface}. This approach creates a smooth surface by interpolating between strategically placed keypoints. The guiding surface $S_g : \{(x,y,z) \in X \mid z = f_{S_g}(x, y)\}$ is constructed using keypoints $B = \{ p_1, p_2, \cdots, p_n \}$ through inverse distance weighting:
    \begin{equation}
        \label{eq:kp_guid_surf}
        f_{S_g}(x, y) = \frac{ \sum_{i=1}^{n} \frac{w_iz_i}{r_i^k}}
        {\sum_{i=1}^{n} \frac{w_i}{r_i^k}}
    \end{equation}
    where $p_i = (x_i, y_i, z_i)$ are the keypoints, $r_i = \| (x, y) - (x_i, y_i) \|$ is the Euclidean distance from point $(x,y)$ to keypoint $i$, $w_i$ is the influence weight of keypoint $p_i$, and $k$ controls the smoothness (higher $k$ creates sharper transitions).
    For FRP, we set $p_1$ and $p_n$ as the start and goal positions, with intermediate keypoints sampled along the connecting line to create a natural pathway. As visualized in Fig. \ref{fig:harmonic_guiding_surface}(a), this creates a smooth surface that naturally guides the leg motion around obstacles.

    \textbf{Convolutional Guiding Surface}. For batch foothold reachability checks where specific goal positions are not predetermined, we employ a terrain-adaptive approach. The guiding surface smooths local terrain variations through convolution with the ground elevation map:
    \begin{equation}
        \label{eq:conv_guid_surf}
        f_{S_g}(x, y) = \frac{\sum_{i,j=1}^nf_G\left(x+(i-\frac{n+1}{2})l, y+(j-\frac{n+1}{2})l\right)}{n^2}
    \end{equation}
    where $n$ is the convolution kernel size (odd number), $l$ is the grid spacing, and $f_G$ is the ground elevation function. This averaging operation removes sharp terrain features that could fragment the feasible domain, creating a smoother surface for more robust reachability analysis as shown in Fig. \ref{fig:harmonic_guiding_surface}(b).
\end{revisebox}

\subsection{Intersection Border Search}
\label{section:intersect_border}

After $S_g$ and $S_a$ are obtained, it becomes possible to search the intersection border, representing $S_a \cap D_f$ in the 2D form that comes with better performance.
As is shown in Fig. \ref{fig:intersection_border_search},
we search the intersection border directly on the grid map for simplicity.
\begin{figure}[h]
    \centering
    \includegraphics[width=1.0\linewidth]{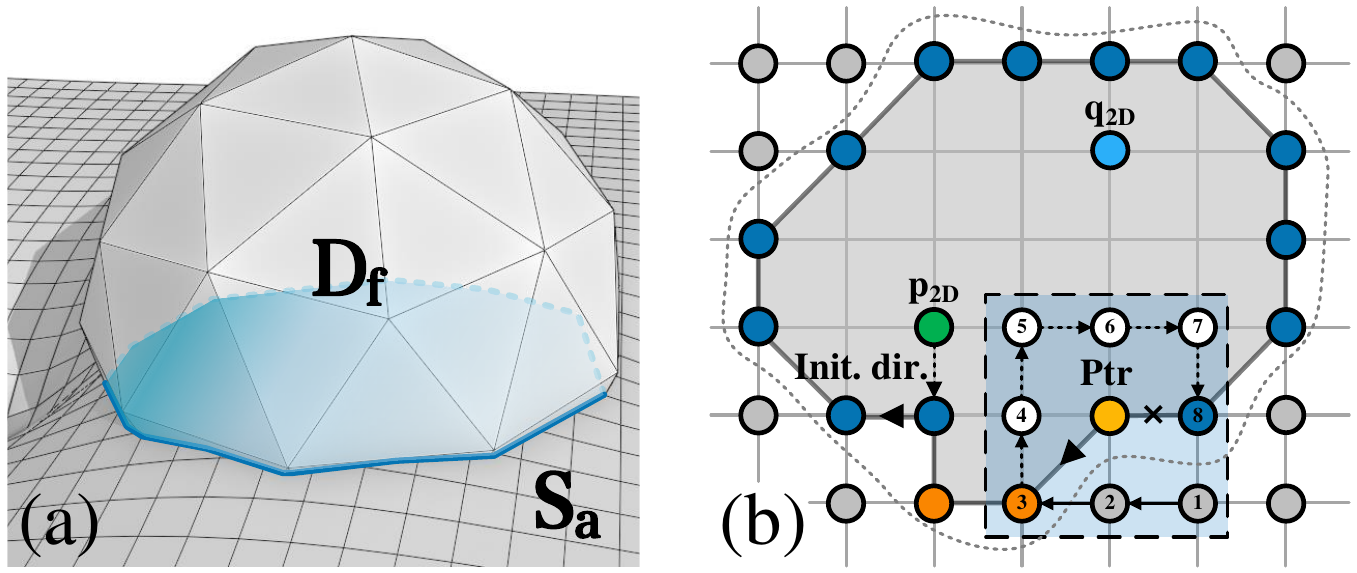}
    \caption{2D intersection border search. (a) Intersection of $D_f$ and $S_a$. (b) schematic diagram of Algorithm \ref{alg:intersect_border_search}.}
    \label{fig:intersection_border_search}
\end{figure}

Algorithm \ref{alg:intersect_border_search} computes intersection border $B$ given the feasible domain $D_f$, auxiliary surface $f_{S_a}$, start point $p$ and goal point $q$.

\begin{algorithm}[h]
    \caption{2D Intersection Border Search}
    \label{alg:intersect_border_search}
    \renewcommand{\algorithmicrequire}{\textbf{Input:}}
    \renewcommand{\algorithmicensure}{\textbf{Output:}}
    \begin{algorithmic}[1]
        \Require $D_f$, $f_{S_a}$, $p$, $q$  
        \Ensure $B$                   
        \State Find nearest grid index $p_{2D}, q_{2D}$ of $p,q$.
        \State $Ptr \leftarrow$ First border point index.
        \State $B$.append($Ptr$)
        \While{$Ptr \neq B[0]$ or $\mathrm{len}(B) < 3$}
        \For{$\delta \in$ index around $Ptr$ clockwise}
        \If {$\mathrm{3DPos}(\delta, f_{S_a})$ is in $D_f$ for the first time}
        \State $Ptr \leftarrow \delta$
        \State $B$.append($Ptr$)
        \State \textbf{break}
        \EndIf
        \EndFor
        \EndWhile
        \State \Return $B$
    \end{algorithmic}
\end{algorithm}

\subsection{Reachability Check}
\label{section:reachability_check}

\begin{figure}[h]
    \centering
    \includegraphics[width=0.8\linewidth]{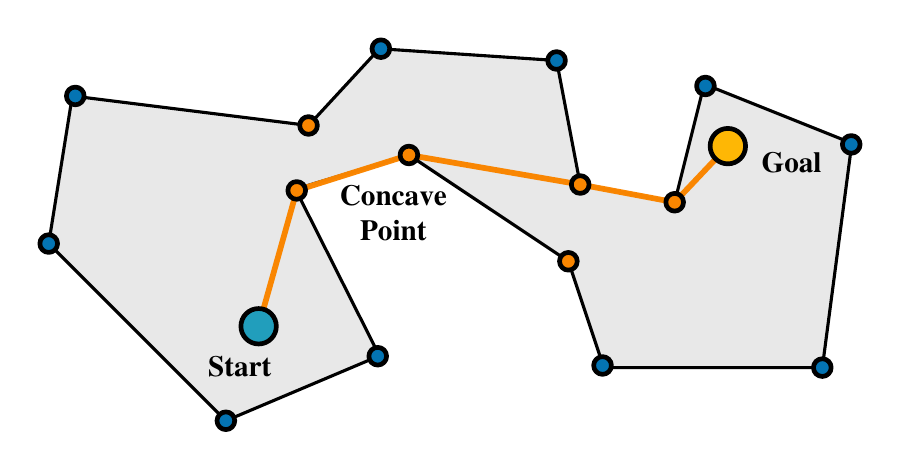}
    \caption{Shortest path in projected 2D feasible domain. The shortest path in this domain goes through concave points and connects start and goal points.}
    \label{fig:shortest_path_astar}
\end{figure}

Visibility graph is introduced for reachability judgment.
As is shown in Fig. \ref{fig:shortest_path_astar}, it can be proved that the shortest path from start to goal consists of line segments among start $p$, goal $q$ and concave points of the projected 2D connected domain, besides, adjacent points on the shortest path must be visible to each other.
If the goal in the projected 2D connected domain is reachable, then the shortest path connecting $p$ and $q$ exists.
Therefore, the goal is \textbf{reachable} if and only if the goal $q$ is visible to start point $p$ or any concave point on the intersection border $B$.
Algorithm \ref{alg:kcfrc} describes the major steps of KCFRC.

\begin{algorithm}[h]
    \caption{Kinematic Collision-Aware FRC}
    \label{alg:kcfrc}
    \renewcommand{\algorithmicrequire}{\textbf{Input:}}
    \renewcommand{\algorithmicensure}{\textbf{Output:}}
    \begin{algorithmic}[1]
        \Require $p$, $q$, $P_p$, $P_q$, $i$, $f_G$, $f_C$,   
        \Ensure bool                   
        \State Calculate SDF of $f_G$ and $f_C$ \Comment{Section \ref{section:coll_detect}}
        \State Generate $S_g$ \Comment{Eq. (\ref{eq:kp_guid_surf}) or (\ref{eq:conv_guid_surf})}
        \State Obtain $f_{S_a}$ from $f_{S_g}, f_G, f_C$ \Comment{Eq. (\ref{eq:aux_surf})}
        \State Construct $D_f^{PITD}$  \Comment{Eq. (\ref{eq:pitd})}
        \State $B \leftarrow \mathrm{IntersectBorder}(D_f^{PITD}, f_{S_a}, p, q)$ \Comment{Alg. \ref{alg:intersect_border_search}}
        \State Obtain concave points $P_{\mathrm{concave}}$ from $B$
        \State Set up $\mathrm{VisGraph}(p_{2D}, q_{2D}, P_{\mathrm{concave}})$ \Comment{Section \ref{section:reachability_check}}
        \For{$x \in \{p_{2D}\} \cup P_{\mathrm{concave}}$}
        \If {$\mathrm{Visible}(x, q_{2D})$}
        \State \Return \textbf{True}
        \EndIf
        \EndFor
        \State \Return \textbf{False}
    \end{algorithmic}
\end{algorithm}

%% file: sections/5_experiments.tex
\section{Experiment Results}
\label{section:experiment}

To illustrate the performance of KCFRC, we conduct an experiment that examines its time performance and scalability, and compare it to other methods.
All experiments were conducted on an Intel NUC12WSKi7 (Intel® Core™ i7-1260P) running Ubuntu 20.04 LTS and ROS Noetic. The algorithms were executed entirely on the CPU without any GPU acceleration.

\subsection{Time Performance and Scalability}

\begin{figure}
   \centering
   \includegraphics[width=1\linewidth]{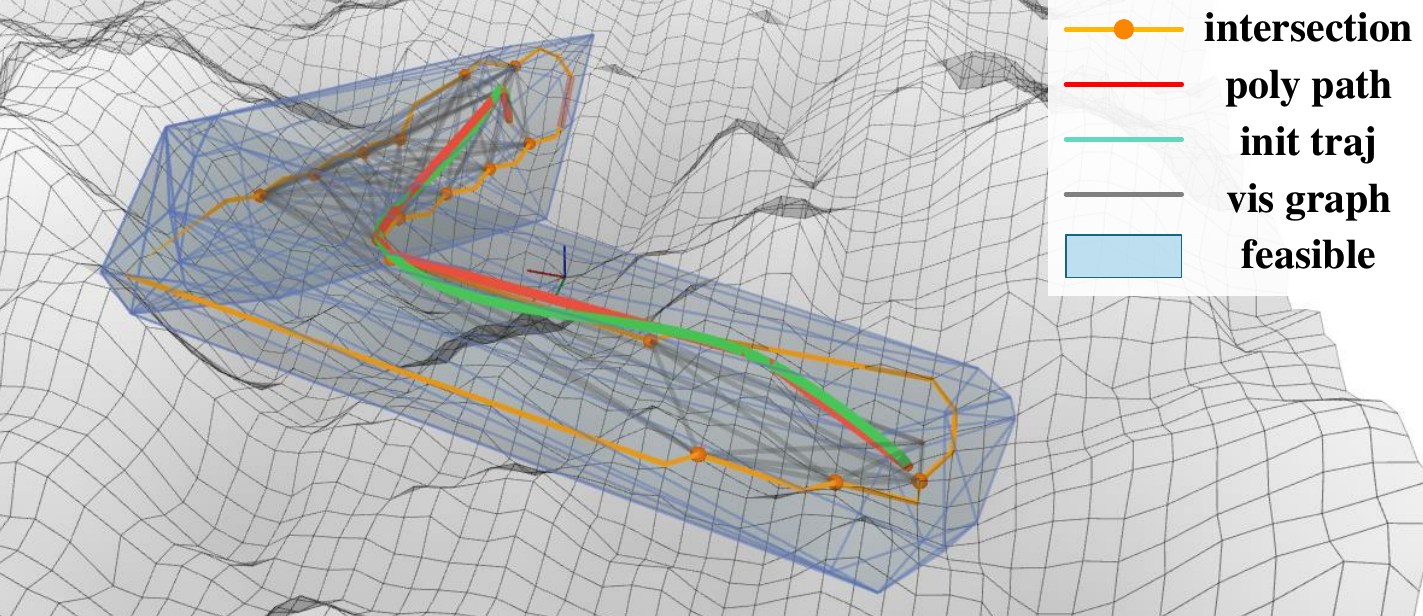}
   \caption{Reachability check experiment. Random convex polyhedra corridor is generated as the feasible domain, and stochastic fractal grid map as obstacles. Meanwhile, the shortest polygon path, visibility graph, and initialized trajectory are visualized.}
   \label{fig:reachability_experiment}
\end{figure}

\begin{figure}
   \centering
   \includegraphics[width=1\linewidth]{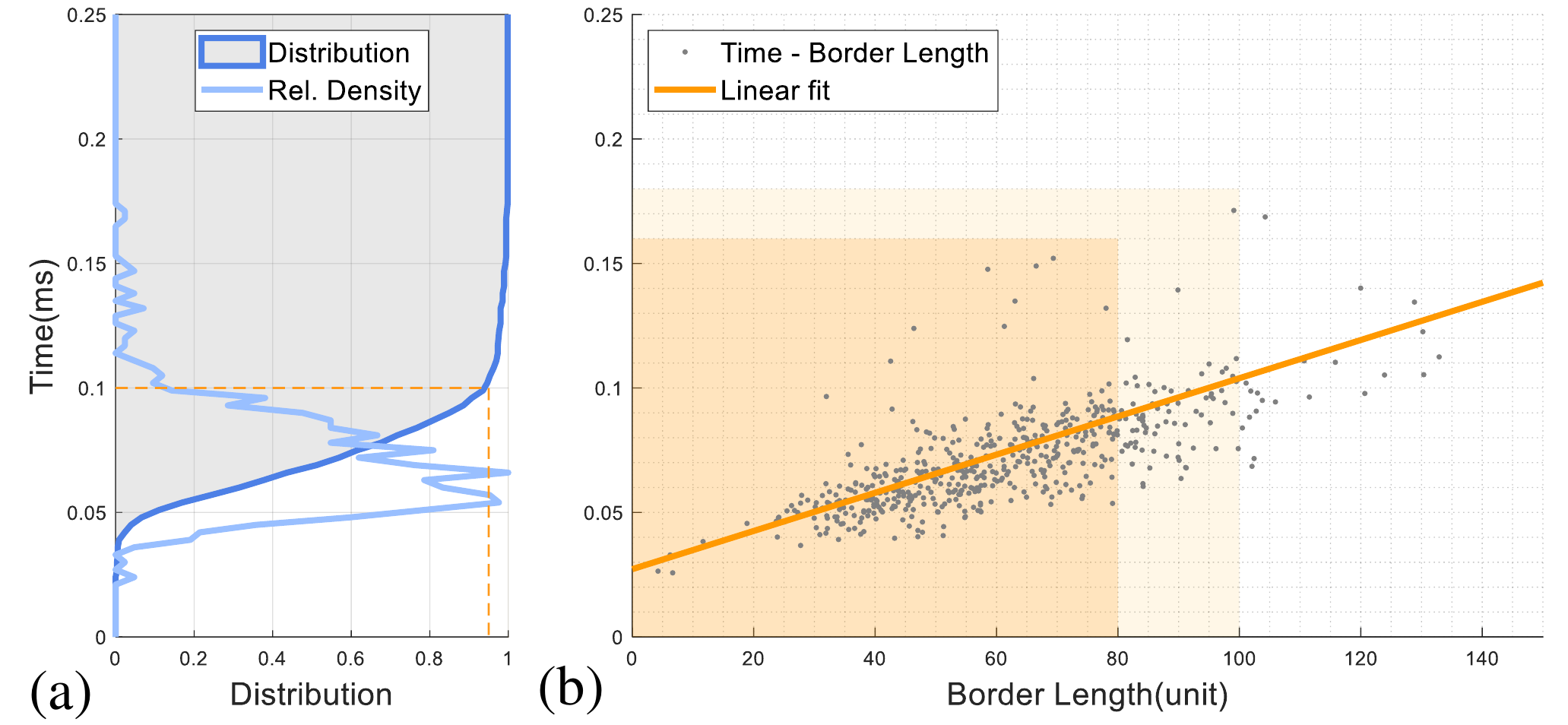}
   \caption{Reachability check time w.r.t intersection border length. Linear fitting of the data is $y=ax+b, a=0.00077, b=0.02722$.}
   \label{fig:reachability_judge_perf}
\end{figure}

Our algorithm is tested in generated scenarios to assess KCFRC performance. Fig. \ref{fig:reachability_experiment} shows one experiment scenario with a $60 \times 60$ fractal noise grid map and randomly generated feasible domain using 3D QuickHull\cite{1996Barber_QuickHull}. Start and goal points are randomly selected at corridor ends for reachability checking, with kinematic and leg collision checks replaced by convex hull checks.

We evaluate computational performance across 1,000 test cases, recording intersection border length and processing time. Fig. \ref{fig:reachability_judge_perf}(a) shows 95\% of cases resolve within 0.1 ms (0.04-0.1 ms typical range). Fig. \ref{fig:reachability_judge_perf}(b) demonstrates linear scaling with border size, with computation consistently under 0.16 ms for border lengths below 80 units and maximum time below 0.2 ms, meeting real-time requirements.

\subsection{Foothold Reachability Check Comparison}

\begin{figure}
   \centering
   \begin{revisebox_nobreak}{4.5, 4.8}
      \includegraphics[width=1\linewidth]{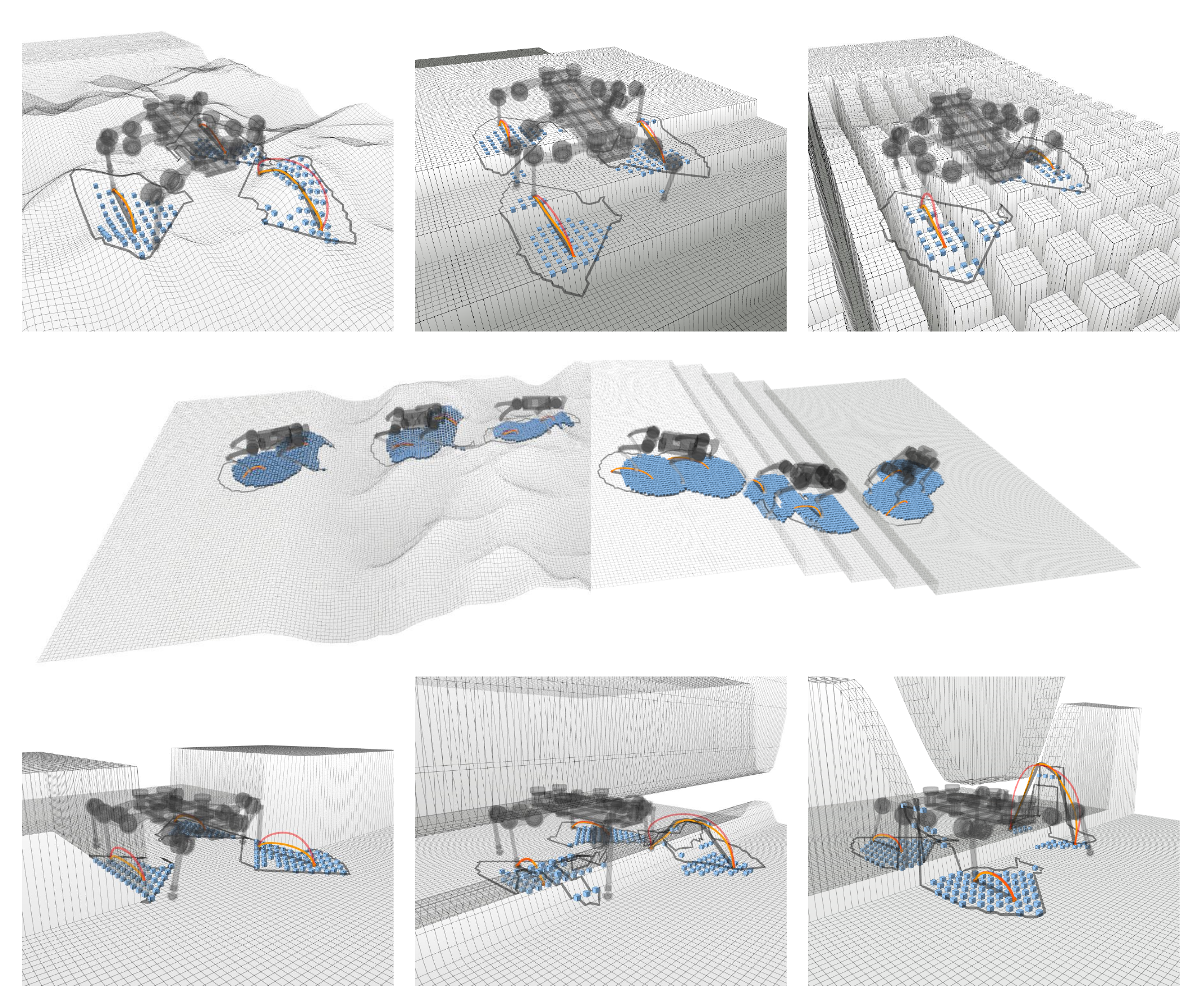}
      \caption{Representative planning scenes for foothold reachability check comparison. \textbf{Top~Left\&Mid:} Hexapod dense planning scenes. \textbf{Top~Right:} Hexapod sparse planning scenes. \textbf{Mid-Row:} Quadruped dense planning scenes. \textbf{Bottom~Row:} Hexapod confined planning scenes. Planned contact states are recorded for foothold reachability check benchmarking. The grid map resolution is set to 2.5 cm/grid.}
      \label{fig:exp_rc_scenes}
   \end{revisebox_nobreak}
\end{figure}

\begin{figure}
   \centering
   \includegraphics[width=1\linewidth]{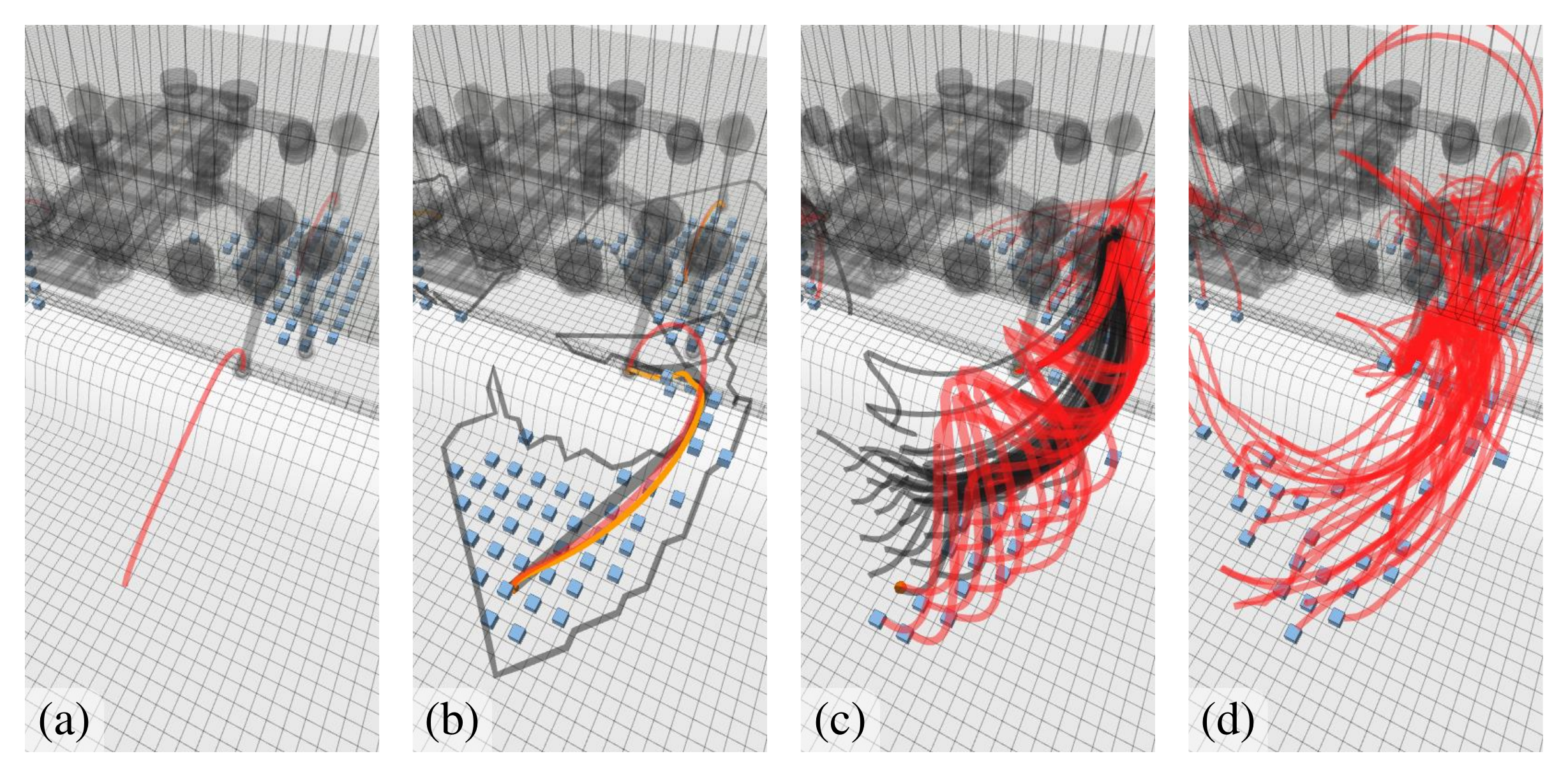}
   \caption{Foothold reachability check methods participate in the reachability check experiment for comparison. (a) FEC with predefined swing trajectory. (b) KCFRC. (c) RRT-Connect. (d) STOMP.}
   \label{fig:exp_rc_planners}
\end{figure}



\subsubsection{Planning Scenes Setup}
To systematically evaluate KCFRC's performance in foothold reachability check,
we benchmark different methods on ElSpider 4 Air hexapod \revise{4.5, 4.8}{and Unitree A1 quadruped} in several representative planning scenes (Fig.~\ref{fig:exp_rc_scenes}).
For each scene, contact states are precomputed by the MCTS planning pipeline\cite{2022Xu_MCTS} and stored. During execution, these states are reloaded to perform foothold reachability checks. For every swing leg transition between consecutive contact states, we sample \textbf{900 candidate footholds} (arranged in a 30$\times$30 structured grid with 0.025~m resolution) around the nominal foothold position of the target state, ensuring total coverage of legs' task space.

\subsubsection{Methods for Comparison}
As shown in Fig.~\ref{fig:exp_rc_planners}, the methods participating in this experiment are as follows, with their key parameters listed in Table \ref{tab:crucial_param}:
\begin{itemize}
   \item \textbf{KCFRC (keypoint-weighted guiding surface)}. \revise{3.2}{Provides higher accuracy by updating the intersection border for every candidate foothold. It is suitable for cluttered or confined environments where accuracy is critical.}
   \item \textbf{KCFRC (convolutional guiding surface)}. \revise{3.2}{Optimized for speed by updating the intersection border once per leg. It is ideal for sparse or real-time applications where computational efficiency is paramount.}
   \item \textbf{RRT-Connect}\cite{2000Kuffner_RRTConnect}. We implement a collision validator and enforce monotonic time-increasing constraints in the configuration space. If a valid path is generated within the time limit, the target foothold is considered reachable.
   \item \textbf{STOMP}\cite{2011Kalakrishnan_STOMP}. We set up STOMP to optimize the trajectory initialized from a linear interpolated one in configuration space, applying the constraints mentioned in previous sections as soft constraints. If the algorithm ends up with a valid trajectory, the candidate foothold is deemed reachable.
   \item \textbf{FEC}\cite{2023fahmi_ViTAL}. We examine the predefined swing trajectory according to FEC. The swing trajectory clears the maximum sampled grid map height on the line connecting adjacent footholds. If no violation of FEC is detected, it is regarded as a reachable foothold.
\end{itemize}

\begin{table}[htbp]
   \centering
   \caption{Key Parameters of the Methods}
   \label{tab:crucial_param}
   \begin{tabular}{lcl} 
      \toprule
      \textbf{Parameters} & \textbf{Value} & \textbf{Description}                      \\
      \midrule
      \textbf{Robot}                                                                   \\
      ElAir Knee          & 0.05 m         & Knee collision ball radius.               \\
      ElAir Foot          & 0.03 m         & Foot collision ball radius.               \\
      A1 Knee             & 0.03 m         & Knee collision ball radius.               \\
      A1 Foot             & 0.02 m         & Foot collision ball radius.               \\
      \midrule
      \textbf{Ground Truth}                                                            \\
      Traj Space          & -              & Configuration space.                      \\
      Max Step            & 0.6            & Max search step in configuration space.   \\
      Max Time            & 10s           & Max search time.                          \\
      Max Retry           & 5             & Retry till success up to 5 times.        \\
      \midrule
      \textbf{RRT-Connect}                                                             \\
      Traj Space          & -              & Configuration space.                      \\
      Max Step            & 0.4            & Max search step in configuration space.   \\
      Max Time            & 1ms/50us       & Max search time.                          \\
      \midrule
      \textbf{STOMP}                                                                   \\
      Traj Space          & -              & Configuration space.                      \\
      Resolution          & 20             & Trajectory resolution in optimization.    \\
      Max Iters           & 100            & Max iterations.                           \\
      Rollout Num         & 10             & Number of rollouts for sampling.          \\
      \midrule
      \textbf{FEC}                                                                     \\
      Traj Space          & -              & Euclidean space.                          \\
      Traj Gen            & -              & 2-Stage Hermite (apex clears max height). \\
      Resolution          & 50             & Resolution of criteria evaluation.        \\
      \bottomrule
   \end{tabular}
\end{table}
While all the methods nominally address kinematic feasibility, foot collision, and leg collision constraints, no existing approach provides rigorous reachability verification.
\revise{2.1, 4.4, 4.7}{To compare these methods' accuracy, precision and recall rates, we establish an \textbf{Exhaustive RRT baseline} as the \textbf{Ground Truth} method (see Table \ref{tab:crucial_param}). This exhaustive RRT provides a probabilistically complete reference by allowing extensive search time (up to 10s per foothold with max 5 retries) to thoroughly explore the configuration space. The method ensures reliable ground truth establishment by eliminating concerns about circular validation, as it operates independently from our proposed KCFRC approach using fundamentally different search principles.}

\begin{table}[h]
   \centering
   \begin{revisebox_nobreak}{4.5, 4.8}
      \caption{Foothold Reachability Check Comparison}
      \label{tab:rc_cmp}
      \input{figures/6_experiments/rc_cmp_table_v2}
   \end{revisebox_nobreak}
\end{table}

\subsubsection{Performance Metrics}
Through all the planning scenes and planned contact states mentioned above, computation time and binary reachability matrices ($30 \times 30$) are collected for each leg in a single contact state transition.
Table \ref{tab:rc_cmp} shows the time performance (per leg, 900 candidate footholds) and the accuracy, precision, and recall rates (of all footholds) of these methods in confined, dense, and sparse scenes respectively.

\subsubsection{Method Comparison}
\begin{revisebox}{4.5, 4.8}
   The experimental results demonstrate KCFRC's effectiveness across different scene types and robot morphologies. KCFRC(key) consistently achieves high accuracy levels (99.4-99.8\%) comparable to RRT-Connect(1ms) (99.6-99.9\%) while maintaining excellent recall rates (98.2-98.7\% vs. 98.9-99.2\%), validating KCFRC's theoretical foundation as both methods can identify nearly all reachable footholds with similar reliability. However, the computational efficiency advantage is substantial: KCFRC(key) operates 6-18× faster than RRT-Connect(1ms) across all scenarios (27.4-50.6ms vs. 90.9-587.4ms), while KCFRC(conv) achieves remarkable speed improvements of 100-400× (1.6-2.6ms) with only modest accuracy degradation (97.8-99.7\% accuracy, 93.4-97.8\% recall). In contrast, FEC demonstrates consistent limitations across challenging terrains, achieving only 76.5-90.0\% recall despite perfect precision, confirming its inadequacy for complex environments where predefined trajectories frequently fail.

   Morphology adaptability is confirmed through quadruped validation, where KCFRC maintains high performance (99.4\% accuracy, 98.2\% recall for key variant) while benefiting from reduced kinematic complexity (33.6ms vs. 47.5ms for hexapod dense scenes). These results establish KCFRC(key) as the optimal choice for accuracy-critical applications requiring reliability comparable to exhaustive planners, while KCFRC(conv) serves real-time applications where computational efficiency is paramount, delivering 100-400× speedup over traditional methods while maintaining $>$97\% accuracy.
\end{revisebox}

\revise{3.5}{FEC's poor performance in confined spaces stems from its reliance on predefined swing trajectories that follow simplistic apex-clearing rules. In complex environments (barriers, tunnels), these rigid trajectories frequently collide with obstacles even when alternative feasible paths exist, explaining the 15-20\% recall degradation compared to KCFRC's adaptive path planning approach.}

\subsubsection{Advantages \& Limitations}
KCFRC demonstrates several significant advantages over existing methods:
\begin{itemize}
   \item \textbf{Time Efficiency.} It outperforms all other methods in the time efficiency index, achieving microsecond-scale computation.
   \item \textbf{High Accuracy.} KCFRC(conv) shows only marginal degradation from KCFRC(key) while maintaining $>$97\% accuracy.
   \item \textbf{Complex Environment Adaptation.} Although confined environments seriously affect FEC accuracy, KCFRC maintains the same level as planner-based methods \revise{3.5}{through its ability to identify multiple feasible pathways rather than being constrained to predefined trajectories.}
\end{itemize}

\begin{revisebox}{2.1, 4.4}
   As a sufficient condition approach, KCFRC may occasionally classify reachable footholds as unreachable (false negatives) in specific scenarios:
   \begin{itemize}
      \item \textbf{Narrow Configuration Tunnels}: When feasible paths involve very narrow passages that guiding surface smoothing cannot adequately represent.
      \item \textbf{Multi-Modal Solutions}: In scenarios with multiple disconnected feasible regions, alternative pathways requiring significant configuration space detours may be missed.
      \item \textbf{Guiding Surface Constraints}: Generation failures for subtle terrain features in highly irregular or fractal-like obstacle configurations.
      \item \textbf{Resolution Dependencies}: Conservative margin settings may classify trajectories as infeasible at discretized domain boundaries where narrow gaps theoretically allow passage.
   \end{itemize}
   These limitations result in false negative rates of 1-6\% depending on scene complexity, robot platform and variant choice while maintaining zero false positives for safety-critical applications.
\end{revisebox}


%% file: figures/6_experiments/rc_cmp_table_v2.tex

\resizebox{\columnwidth}{!}{%
  \begin{tabular}{@{}lrrccc@{}}
    \toprule
    Method                                                                      &
    \multicolumn{1}{c}{\begin{tabular}[c]{@{}c@{}}Ave.\\ Time(ms)\end{tabular}} &
    \multicolumn{1}{c}{\begin{tabular}[c]{@{}c@{}}Max.\\ Time(ms)\end{tabular}} &
    \begin{tabular}[c]{@{}c@{}}Ave.\\ Accuracy\end{tabular}                     &
    \begin{tabular}[c]{@{}c@{}}Ave.\\ Precision\end{tabular}                    &
    \begin{tabular}[c]{@{}c@{}}Ave.\\ Recall\end{tabular}                         \\ \midrule
    \multicolumn{6}{l}{\textbf{Hexapod Confined Planning Scenes}}                 \\
    {\color[HTML]{656565} Ground Truth}                                         &
    {\color[HTML]{656565} 827.852}                                              &
    {\color[HTML]{656565} 18941.200}                                            &
    {\color[HTML]{656565} 1.000}                                                &
    {\color[HTML]{656565} 1.000}                                                &
    {\color[HTML]{656565} 1.000}                                                  \\
    KCFRC(key)                                                                  &
    50.574                                                                      &
    78.740                                                                      &
    {\color[HTML]{000000} \textbf{0.996}}                                       &
    0.998                                                                       &
    {\color[HTML]{000000} \textbf{0.986}}                                         \\
    KCFRC(conv)                                                                 &
    {\color[HTML]{F56B00} \textbf{2.580}}                                       &
    {\color[HTML]{F56B00} \textbf{4.070}}                                       &
    0.993                                                                       &
    \textbf{0.999}                                                              &
    0.974                                                                         \\
    RRT-C(1ms)                                                                  &
    303.257                                                                     &
    1085.290                                                                    &
    {\color[HTML]{F56B00} \textbf{0.998}}                                       &
    {\color[HTML]{F56B00} \textbf{1.000}}                                       &
    {\color[HTML]{F56B00} \textbf{0.992}}                                         \\
    RRT-C(50us)                                                                 &
    209.600                                                                     &
    948.055                                                                     &
    0.972                                                                       &
    {\color[HTML]{F56B00} \textbf{1.000}}                                       &
    0.889                                                                         \\
    STOMP                                                                       &
    796.161                                                                     &
    5099.620                                                                    &
    0.976                                                                       &
    \textbf{0.999}                                                              &
    0.908                                                                         \\
    FEC                                                                         &
    \textbf{10.226}                                                             &
    \textbf{36.722}                                                             &
    0.948                                                                       &
    {\color[HTML]{F56B00} \textbf{1.000}}                                       &
    0.796                                                                         \\ \midrule
    \multicolumn{6}{l}{\textbf{Hexapod Dense Planning Scenes}}                    \\
    {\color[HTML]{656565} Ground Truth}                                         &
    {\color[HTML]{656565} 458.751}                                              &
    {\color[HTML]{656565} 5193.540}                                             &
    {\color[HTML]{656565} 1.000}                                                &
    {\color[HTML]{656565} 1.000}                                                &
    {\color[HTML]{656565} 1.000}                                                  \\
    KCFRC(key)                                                                  &
    47.467                                                                      &
    83.766                                                                      &
    {\color[HTML]{000000} \textbf{0.996}}                                       &
    {\color[HTML]{F56B00} \textbf{1.000}}                                       &
    {\color[HTML]{000000} \textbf{0.987}}                                         \\
    KCFRC(conv)                                                                 &
    {\color[HTML]{F56B00} \textbf{2.384}}                                       &
    {\color[HTML]{F56B00} \textbf{5.295}}                                       &
    0.994                                                                       &
    {\color[HTML]{F56B00} \textbf{1.000}}                                       &
    0.977                                                                         \\
    RRT-C(1ms)                                                                  &
    181.670                                                                     &
    460.035                                                                     &
    {\color[HTML]{F56B00} \textbf{0.998}}                                       &
    {\color[HTML]{F56B00} \textbf{1.000}}                                       &
    {\color[HTML]{F56B00} \textbf{0.990}}                                         \\
    RRT-C(50us)                                                                 &
    90.465                                                                      &
    333.077                                                                     &
    0.972                                                                       &
    {\color[HTML]{F56B00} \textbf{1.000}}                                       &
    0.892                                                                         \\
    STOMP                                                                       &
    309.474                                                                     &
    1891.580                                                                    &
    0.992                                                                       &
    {\color[HTML]{F56B00} \textbf{1.000}}                                       &
    0.968                                                                         \\
    FEC                                                                         &
    \textbf{9.494}                                                              &
    \textbf{35.968}                                                             &
    0.943                                                                       &
    {\color[HTML]{F56B00} \textbf{1.000}}                                       &
    0.782                                                                         \\ \midrule
    \multicolumn{6}{l}{\textbf{Quadruped Dense Planning Scenes}}                  \\
    {\color[HTML]{656565} Ground Truth}                                         &
    {\color[HTML]{656565} 654.656}                                              &
    {\color[HTML]{656565} 7022.720}                                             &
    {\color[HTML]{656565} 1.000}                                                &
    {\color[HTML]{656565} 1.000}                                                &
    {\color[HTML]{656565} 1.000}                                                  \\
    KCFRC(key)                                                                  &
    33.559                                                                      &
    58.345                                                                      &
    \textbf{0.994}                                                              &
    {\color[HTML]{F56B00} \textbf{1.000}}                                       &
    \textbf{0.982}                                                                \\
    KCFRC(conv)                                                                 &
    {\color[HTML]{F56B00} \textbf{1.909}}                                       &
    {\color[HTML]{F56B00} \textbf{4.536}}                                       &
    0.978                                                                       &
    {\color[HTML]{F56B00} \textbf{1.000}}                                       &
    0.934                                                                         \\
    RRT-C(1ms)                                                                  &
    587.400                                                                     &
    2004.230                                                                    &
    {\color[HTML]{F56B00} \textbf{0.996}}                                       &
    {\color[HTML]{F56B00} \textbf{1.000}}                                       &
    {\color[HTML]{F56B00} \textbf{0.989}}                                         \\
    RRT-C(50us)                                                                 &
    470.874                                                                     &
    1706.380                                                                    &
    0.962                                                                       &
    {\color[HTML]{F56B00} \textbf{1.000}}                                       &
    0.889                                                                         \\
    STOMP                                                                       &
    966.210                                                                     &
    4507.900                                                                    &
    {\color[HTML]{000000} 0.985}                                                &
    {\color[HTML]{F56B00} \textbf{1.000}}                                       &
    {\color[HTML]{000000} 0.955}                                                  \\
    FEC                                                                         &
    \textbf{5.306}                                                              &
    \textbf{12.778}                                                             &
    0.920                                                                       &
    {\color[HTML]{F56B00} \textbf{1.000}}                                       &
    0.765                                                                         \\ \midrule
    \multicolumn{6}{l}{\textbf{Hexapod Sparse Planning Scenes}}                   \\
    {\color[HTML]{656565} Ground Truth}                                         &
    {\color[HTML]{656565} 103.470}                                              &
    {\color[HTML]{656565} 352.172}                                              &
    {\color[HTML]{656565} 1.000}                                                &
    {\color[HTML]{656565} 1.000}                                                &
    {\color[HTML]{656565} 1.000}                                                  \\
    KCFRC(key)                                                                  &
    27.376                                                                      &
    68.729                                                                      &
    \textbf{0.998}                                                              &
    {\color[HTML]{F56B00} \textbf{1.000}}                                       &
    0.984                                                                         \\
    KCFRC(conv)                                                                 &
    {\color[HTML]{F56B00} \textbf{1.627}}                                       &
    {\color[HTML]{F56B00} \textbf{2.748}}                                       &
    0.997                                                                       &
    {\color[HTML]{F56B00} \textbf{1.000}}                                       &
    0.978                                                                         \\
    RRT-C(1ms)                                                                  &
    90.888                                                                      &
    350.998                                                                     &
    {\color[HTML]{F56B00} \textbf{0.999}}                                       &
    {\color[HTML]{F56B00} \textbf{1.000}}                                       &
    \textbf{0.992}                                                                \\
    RRT-C(50us)                                                                 &
    46.812                                                                      &
    167.921                                                                     &
    0.984                                                                       &
    {\color[HTML]{F56B00} \textbf{1.000}}                                       &
    0.898                                                                         \\
    STOMP                                                                       &
    74.021                                                                      &
    543.723                                                                     &
    {\color[HTML]{F56B00} \textbf{0.999}}                                       &
    {\color[HTML]{F56B00} \textbf{1.000}}                                       &
    {\color[HTML]{F56B00} \textbf{0.994}}                                         \\
    FEC                                                                         &
    \textbf{6.070}                                                              &
    \textbf{14.660}                                                             &
    0.984                                                                       &
    {\color[HTML]{F56B00} \textbf{1.000}}                                       &
    0.900                                                                         \\ \bottomrule
  \end{tabular}%
}

%% file: sections/6_applications.tex
\section{Applications of KCFRC}
\label{section:applications}

We applied the KCFRC algorithm to the MCTS contact planning pipeline\cite{2022Xu_MCTS} based on ElSpider 4 Air hexapod robot.
As illustrated in Fig. \ref{fig:exp_barrier}(a), the ElSpider 4 Air features six legs, providing a total of 18 degrees of freedom (DOF). The robot has a body length of approximately 0.9m, a standing width of about 0.8m, a body standing height of about 0.24m, and a total assembly weight of 30kg.
For this work, we utilized OptiTrack to acquire robot odometry and an RS Bpearl LiDAR for elevation mapping.

\subsection{KCFRC in Legged Robot Planning}

KCFRC can be used for foothold filtration during contact planning (Fig. \ref{fig:1_reachability_example}). In MCTS contact state planning algorithm, candidate robot states are spawned with respect to robot reachability constraints. KCFRC can be applied to filter out unreachable footholds, eliminating contact states that cause inevitable collisions.

\revise{2.2, 3.1}{Beyond contact planning, KCFRC's computational efficiency enables integration with contemporary control frameworks. For \textbf{MPC integration}, KCFRC can serve as a reference generator, producing kinematically feasible swing trajectories as soft tracking constraints in whole-body optimization. For \textbf{RL integration}, KCFRC enhances learning through reward shaping, where feasible trajectories encourage collision-free swing patterns. Our method complements existing frameworks by providing preprocessing filters that reduce computational burden for subsequent planning stages.}

\begin{figure}
    \centering
    \begin{revisebox_nobreak}{3.3}
        \includegraphics[width=1\linewidth]{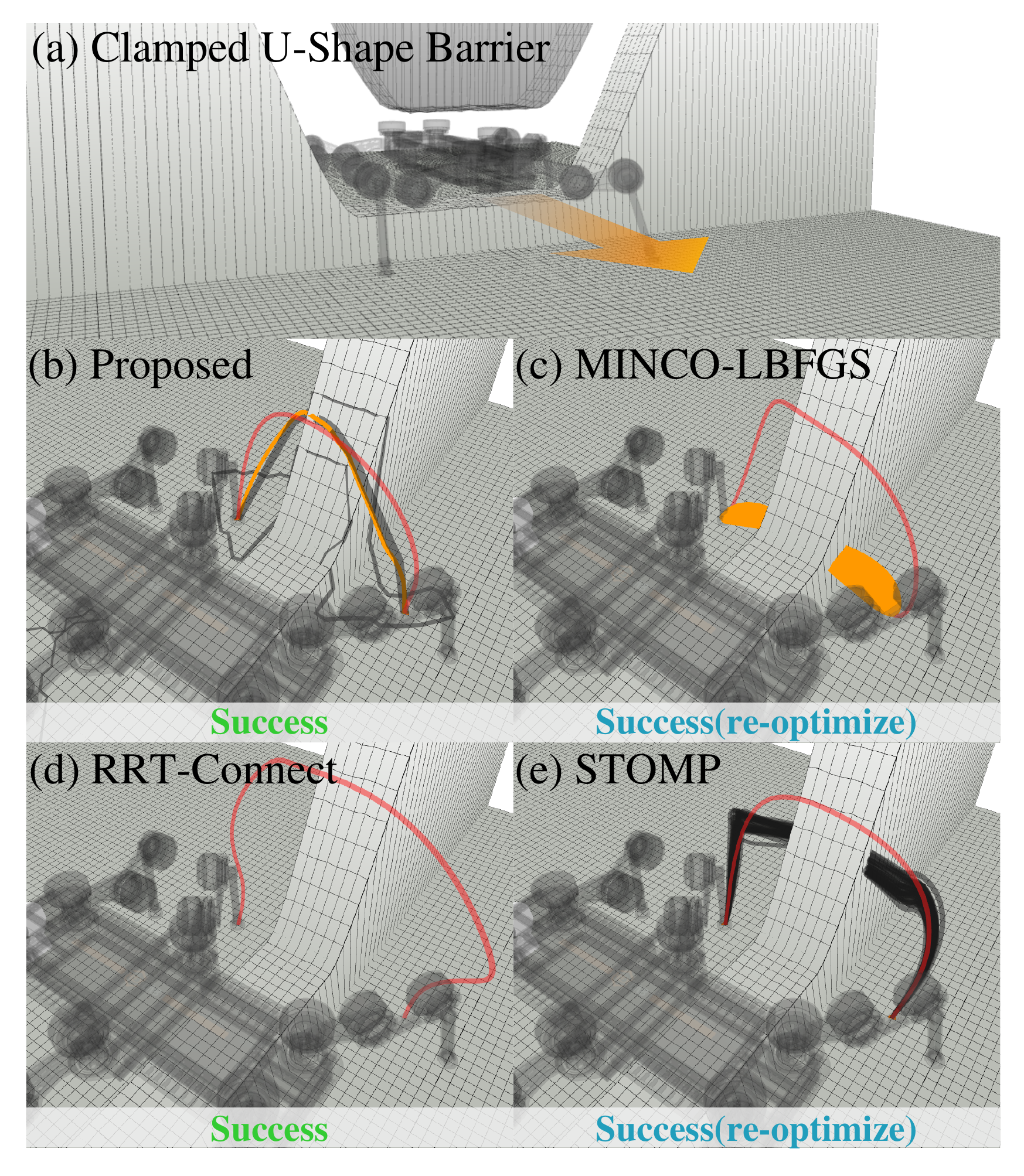}
        \caption{Comparing our proposed method with other methods in a U-shape clamped barrier planning scene considering both knee and foot collision, the final result trajectory is drawn in red. (b) The proposed planner returns a smooth swing trajectory in one shot. (c) MINCO-LBFGS optimizer returns success after many retries. (d) RRT-Connect finds a feasible swing trajectory with low quality even after interpolation. (e) STOMP returns success after many retries.}
        \label{fig:confined_space_cmpl}
    \end{revisebox_nobreak}
\end{figure}

\begin{figure*}[t!]
    \centering
    \includegraphics[width=0.97\textwidth]{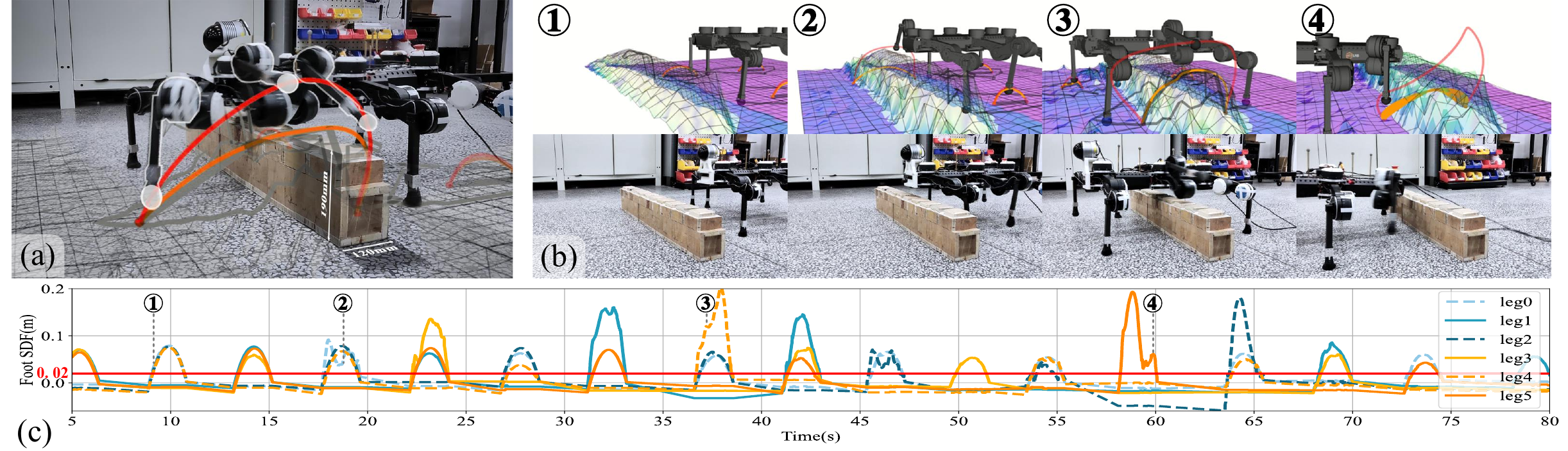}
    \caption{ElSpider 4 Air climbs across a barrier. (a) The robot is getting across a barrier 0.12m in width and 0.19m in height. (b) Snapshots of the robot and Rviz visualization. (c) Foot SDF curves with time of the snapshots annotated. The SDF value of each foot during the swing phase is no less than 0.02m, claiming no unexpected collision.}
    \label{fig:exp_barrier}
\end{figure*}

\subsection{Trajectory Initialization for Optimization Acceleration}
\label{section:trajopt_acc}

\begin{table}[]
    \centering
    \caption{Comparison of swing trajectory planning performance}
    \label{table:planning_cmp}
    \input{figures/6_experiments/table_planners_cmp_v2}
\end{table}

In KCFRC, visibility graph is established for reachability check. Further, a poly trajectory can be initialized from it using popular search algorithms like AStar. The initialized trajectory helps a lot in accelerating the optimization process and alleviate the local minima dilemma.

To demonstrate the capability of KCFRC, we benchmark \textbf{MINCO-LBFGS with KCFRC (proposed)} against \textbf{MINCO-LBFGS}\cite{2022Wang_GCOPTER}, \textbf{RRT-Connect}\cite{2000Kuffner_RRTConnect}, and \textbf{STOMP}\cite{2011Kalakrishnan_STOMP} planners in the scenarios mentioned in the previous section.
We record consumption time, trajectory length and its cost (Acc. square integration in 1s normalized time).
Results are listed in Table. \ref{table:planning_cmp}.

As the statistics show, our method is faster and steadier than any other method, and the planned trajectories take lower costs on average. However, the average length is not the shortest, because it is relevant to lift-off speed and the obstacles.
Although RRT-Connect succeeds in all test cases, some of the trajectories are illegal because the time frames may not be monotone increasing, thus we count it as a failure. This is a fatal problem when deploying RRT-based planners to a space-time optimization problem.

\begin{revisebox}{3.3}
    \subsection{Swing Trajectory Planning in Confined Space}
    We test swing trajectory planners mentioned above in a U-shaped clamped barrier planning scene (Fig. \ref{fig:confined_space_cmpl}(a)). Both knee and foot collisions are considered, and the planning problem is constructed in configuration space.
    For the specific planning problem shown in Fig. \ref{fig:confined_space_cmpl}, the proposed method converges rapidly from the initialized trajectory (Fig. \ref{fig:confined_space_cmpl}(b)), while other gradient-based optimizers like MINCO-LBFGS and STOMP stuck in local minima for many times (Fig. \ref{fig:confined_space_cmpl}(b-c)). Though RRT-Connect finds a feasible trajectory, it is not optimal.
    Apparently, our proposed method generates appropriate initial trajectories so that the local-minima problem can be avoided, and it takes a shorter time to optimize.
\end{revisebox}



\subsection{Real-world Experiments}

To demonstrate the capability of the KCFRC and MCTS contact planning pipeline, we conducted multiple real-world experiments on various terrains.

\subsubsection{Barrier Climbing}
As shown in Fig. \ref{fig:exp_barrier}(a), we tested the robot climbing across a barrier that is 0.12m in width and 0.19m in height, about 50\% of the leg length. The orange swing trajectory is initialized from KCFRC, and the red one is the optimization result. Fig. \ref{fig:exp_barrier}(b) contains snapshots of both Rviz visualization and the real robot through the whole process. Foot SDF curves are drawn in Fig. \ref{fig:exp_barrier}(c). The SDF value means the shortest distance between the foot and the obstacles. The shortest distance of each foot during swing phases is no less than 0.02m, indicating no collisions.

\subsubsection{Stairs and Timber Piles}
The pipeline is also capable of conquering rough terrains like stairs and timber piles. Fig. \ref{fig:exp_hardware}(c1, c2) shows snapshots of the hexapod climbing up and down stairs. Each step is 0.12m high and 0.35m deep. For the timber piles in Fig. \ref{fig:exp_hardware}(a1, a2), each timber is 0.2m in width and length, with height ranging from 0.1m to 0.4m. Due to the difficulty of capturing the terrain morphology only using one Lidar, we constructed the grid map in advance. KCFRC plays an important role in avoiding leg collisions in rough terrains by offering a visibility map to initialize the swing trajectory.

\subsubsection{Hole Navigation}
\revise{4.5, 4.8}{
    The hexapod successfully climbed through a hole as shown in Fig. \ref{fig:exp_hardware}(b1, b2), demonstrating the algorithm's ability to navigate confined spaces with complex geometric constraints.}

\begin{figure}
    \centering
    \begin{revisebox_nobreak}{4.5, 4.8}
        \includegraphics[width=1.0\linewidth]{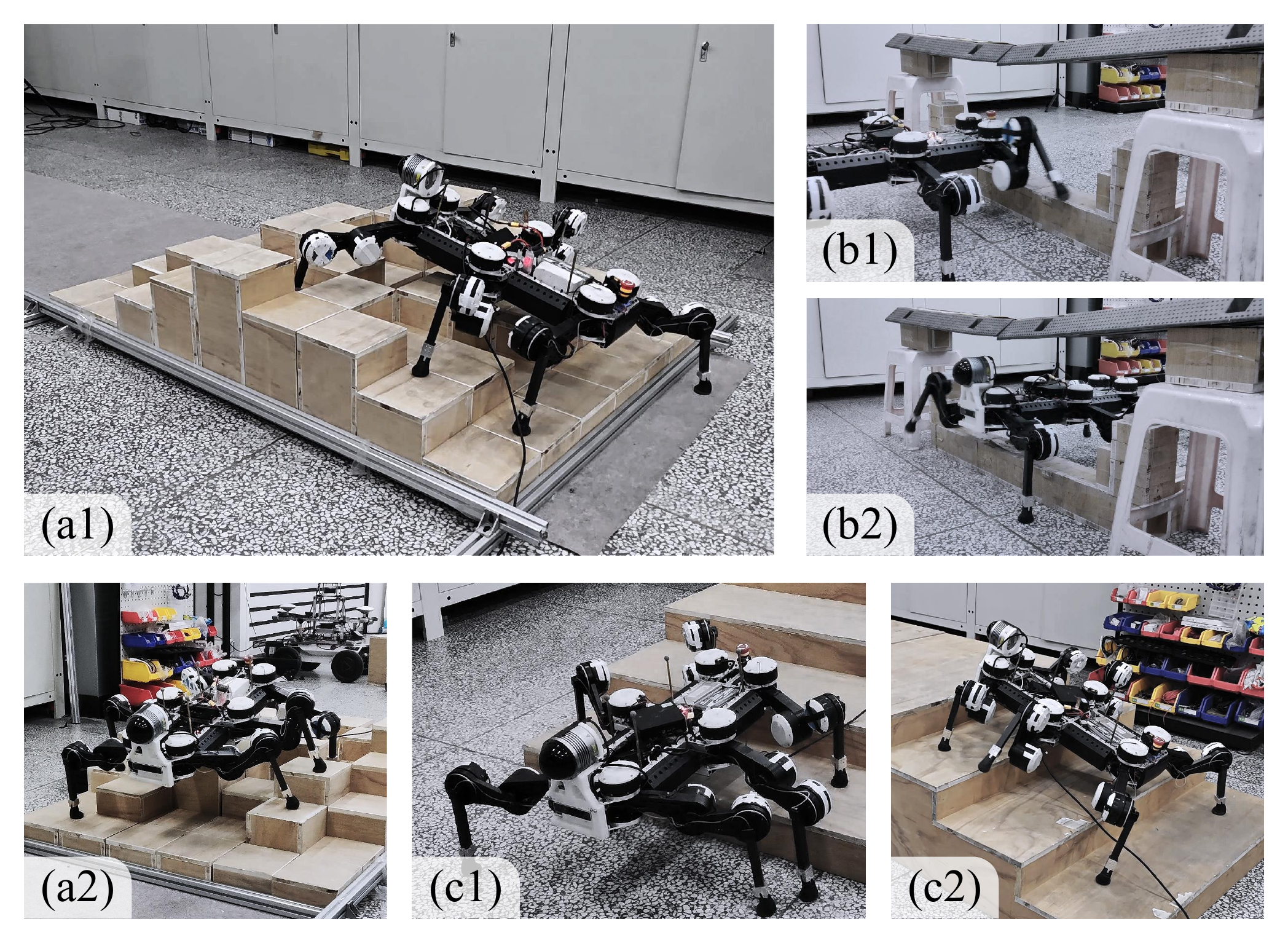}
        \caption{Real-world experimental hardware setups demonstrating KCFRC-equipped hexapod locomotion across diverse terrains. (a1, a2) Timber piles: ElSpider 4 Air navigating across irregular timber obstacles with varying heights from 0.1m to 0.4m. (b1, b2) Hole navigation: Hexapod climbing through confined spaces requiring precise foothold planning and collision avoidance. (c1, c2) Stairs: Robot ascending and descending stair structures with 0.12m step height and 0.35m depth, showcasing multi-level terrain traversal capabilities.}
        \label{fig:exp_hardware}
    \end{revisebox_nobreak}
\end{figure}

%% file: figures/6_experiments/table_planners_cmp_v2.tex
\resizebox{\columnwidth}{!}{%
\begin{tabular}{@{}ccccccc@{}}
\toprule
\multicolumn{2}{c}{\begin{tabular}[c]{@{}c@{}}Planners in\\ Config Space\end{tabular}} &
  \begin{tabular}[c]{@{}c@{}}Time\\ ($ms$)\end{tabular} &
  \begin{tabular}[c]{@{}c@{}}Len.\\ ($rad$)\end{tabular} &
  \begin{tabular}[c]{@{}c@{}}Cost\\ ($s^{-3}$)\end{tabular} &
  \begin{tabular}[c]{@{}c@{}}Succ.\\ Num.\end{tabular} &
  \begin{tabular}[c]{@{}c@{}}Succ.\\ Rate\end{tabular} \\ \midrule
\multirow{3}{*}{\begin{tabular}[c]{@{}c@{}}Proposed\end{tabular}} &
  Mean &
  \textbf{0.818} &
  2.440 &
  \textbf{140.117} &
  \multirow{3}{*}{\textbf{493}} &
  \multirow{3}{*}{\textbf{0.992}} \\
 &
  Max. &
  12.509 &
  8.031 &
  6079.370 &
   &
   \\
 &
  Std. &
  0.871 &
  1.010 &
  365.728 &
   &
   \\ \midrule
\multirow{3}{*}{\begin{tabular}[c]{@{}c@{}}MINCO\\ LBFGS\end{tabular}} &
  Mean &
  1.447 &
  3.296 &
  147.635 &
  \multirow{3}{*}{490} &
  \multirow{3}{*}{0.986} \\
 &
  Max. &
  9.025 &
  5.352 &
  697.719 &
   &
   \\
 &
  Std. &
  0.909 &
  0.647 &
  71.668 &
   &
   \\ \midrule
\multirow{3}{*}{\begin{tabular}[c]{@{}c@{}}RRT\\ Connect\end{tabular}} &
  Mean &
  1.687 &
  3.560 &
  \textgreater{}50000 &
  \multirow{3}{*}{445} &
  \multirow{3}{*}{0.895} \\
 &
  Max. &
  7.158 &
  209.517 &
  \textgreater{}50000 &
   &
   \\
 &
  Std. &
  0.819 &
  11.733 &
  \textgreater{}50000 &
   &
   \\ \midrule
\multirow{3}{*}{STOMP} &
  Mean &
  1.537 &
  \textbf{1.585} &
  408.935 &
  \multirow{3}{*}{485} &
  \multirow{3}{*}{0.976} \\
 &
  Max. &
  30.570 &
  4.397 &
  5281.180 &
   &
   \\
 &
  Std. &
  3.651 &
  0.692 &
  543.638 &
   &
   \\ \bottomrule
\end{tabular}%
}

%% file: sections/7_conclusion.tex
\section{Conclusion}
\label{section:conclusion}


 




This paper addresses the Foothold Reachability Problem for legged robot locomotion in confined environments. We establish a sufficient condition for foothold reachability and implement the Kinematic Collision-Aware Foothold Reachability Check (KCFRC) algorithm, achieving notable computational efficiency. Integration with MCTS planning on the ElSpider Air 4 hexapod demonstrates effective navigation over rough terrains, while trajectory optimization acceleration is achieved through visibility graph initialization.

The proposed method has limitations, particularly regarding grid map resolution accuracy. Since intersection borders are computed on discrete grids, reachability results may miss some feasible points near borders that fall slightly outside grid boundaries. Future improvements include incorporating time dimensions into reachability checks, refining intersection border search for complex connectivity, and deploying the approach on more dynamic controllers.

%% file: library.bib
@misc{2023Balakumar_curobo_report,
  title         = {cuRobo: Parallelized Collision-Free Minimum-Jerk Robot Motion Generation},
  author        = {Balakumar Sundaralingam and Siva Kumar Sastry Hari and Adam Fishman and Caelan Garrett
                   and Karl Van Wyk and Valts Blukis and Alexander Millane and Helen Oleynikova and Ankur Handa
                   and Fabio Ramos and Nathan Ratliff and Dieter Fox},
  year          = {2023},
  eprint        = {2310.17274},
  archiveprefix = {arXiv},
  primaryclass  = {cs.RO}
}

@inproceedings{2019carpentier_pinocchio,
  title     = {The {{Pinocchio C}}++ Library -- {{A}} Fast and Flexible Implementation of Rigid Body Dynamics Algorithms and Their Analytical Derivatives},
  booktitle = {{{IEEE}} International Symposium on System Integrations ({{SII}})},
  author    = {Carpentier, Justin and Saurel, Guilhem and Buondonno, Gabriele and Mirabel, Joseph and Lamiraux, Florent and Stasse, Olivier and Mansard, Nicolas},
  year      = {2019},
  doi       = {10.1109/SII.2019.8700380}
}

@inproceedings{2017Li,
  title     = {Foothold Selection for Quadruped Robot Based on Learning from Expert},
  booktitle = {2017 2nd {{International Conference}} on {{Advanced Robotics}} and {{Mechatronics}} ({{ICARM}})},
  author    = {Li, Xingdong and Li, Jian and Guo, Yanling},
  year      = {2017},
  month     = aug,
  pages     = {223--228},
  doi       = {10.1109/ICARM.2017.8273164},
  urldate   = {2025-01-23}
}

@inproceedings{2010Belter,
  title     = {Rough Terrain Mapping and Classification for Foothold Selection in a Walking Robot},
  booktitle = {2010 {{IEEE Safety Security}} and {{Rescue Robotics}}},
  author    = {Belter, Dominik and Skrzypczy{\'n}ski, Piotr},
  year      = {2010},
  month     = jul,
  pages     = {1--6},
  issn      = {2374-3247},
  doi       = {10.1109/SSRR.2010.5981552},
  urldate   = {2025-01-23}
}

@inproceedings{2019Griffin_Footstep_humanoid,
  title     = {Footstep {{Planning}} for {{Autonomous Walking Over Rough Terrain}}},
  booktitle = {2019 {{IEEE-RAS}} 19th {{International Conference}} on {{Humanoid Robots}} ({{Humanoids}})},
  author    = {Griffin, Robert J. and Wiedebach, Georg and McCrory, Stephen and Bertrand, Sylvain and Lee, Inho and Pratt, Jerry},
  year      = {2019},
  month     = oct,
  pages     = {9--16},
  issn      = {2164-0580},
  doi       = {10.1109/Humanoids43949.2019.9035046},
  urldate   = {2025-01-20}
}

@article{2011Zucker_chomp_littledog,
  title      = {Optimization and Learning for Rough Terrain Legged Locomotion},
  shorttitle = {{{CHOMP-LittleDog}}},
  author     = {Zucker, Matt and Ratliff, Nathan and Stolle, Martin and Chestnutt, Joel and Bagnell, J Andrew and Atkeson, Christopher G and Kuffner, James},
  year       = {2011},
  month      = feb,
  journal    = {The International Journal of Robotics Research},
  volume     = {30},
  number     = {2},
  pages      = {175--191},
  issn       = {0278-3649, 1741-3176},
  doi        = {10.1177/0278364910392608},
  urldate    = {2023-10-17},
  langid     = {english}
}

@inproceedings{2011Kalakrishnan_STOMP,
  title      = {{{STOMP}}: {{Stochastic}} Trajectory Optimization for Motion Planning},
  shorttitle = {{{STOMP}}},
  booktitle  = {2011 {{IEEE International Conference}} on {{Robotics}} and {{Automation}}},
  author     = {Kalakrishnan, Mrinal and Chitta, Sachin and Theodorou, Evangelos and Pastor, Peter and Schaal, Stefan},
  year       = {2011},
  month      = may,
  pages      = {4569--4574},
  issn       = {1050-4729},
  doi        = {10.1109/ICRA.2011.5980280},
  urldate    = {2024-08-17}
}

@inproceedings{2010Kalakrishnan,
  title     = {Fast, Robust Quadruped Locomotion over Challenging Terrain},
  booktitle = {2010 {{IEEE International Conference}} on {{Robotics}} and {{Automation}}},
  author    = {Kalakrishnan, Mrinal and Buchli, Jonas and Pastor, Peter and Mistry, Michael and Schaal, Stefan},
  year      = {2010},
  month     = may,
  pages     = {2665--2670},
  issn      = {1050-4729},
  doi       = {10.1109/ROBOT.2010.5509805},
  urldate   = {2025-01-23}
}

@inproceedings{2008Kolter_LittleDog,
  title     = {A Control Architecture for Quadruped Locomotion over Rough Terrain},
  booktitle = {2008 {{IEEE International Conference}} on {{Robotics}} and {{Automation}}},
  author    = {Kolter, J. Zico and Rodgers, Mike P. and Ng, Andrew Y.},
  year      = {2008},
  month     = may,
  pages     = {811--818},
  publisher = {IEEE},
  address   = {Pasadena, CA, USA},
  doi       = {10.1109/ROBOT.2008.4543305},
  urldate   = {2025-01-23},
  isbn      = {978-1-4244-1646-2},
  langid    = {english}
}

@article{2023Marcucci_GCS,
  title      = {Motion Planning around Obstacles with Convex Optimization},
  shorttitle = {Convex {{Set}} for {{MP}}},
  author     = {Marcucci, Tobia and Petersen, Mark and {von Wrangel}, David and Tedrake, Russ},
  year       = {2023},
  month      = nov,
  journal    = {Science Robotics},
  volume     = {8},
  number     = {84},
  pages      = {eadf7843},
  publisher  = {American Association for the Advancement of Science},
  doi        = {10.1126/scirobotics.adf7843},
  urldate    = {2023-11-24}
}

@misc{2023Petersen_GCS_gen,
  title         = {Growing {{Convex Collision-Free Regions}} in {{Configuration Space}} Using {{Nonlinear Programming}}},
  author        = {Petersen, Mark and Tedrake, Russ},
  year          = {2023},
  month         = mar,
  number        = {arXiv:2303.14737},
  eprint        = {2303.14737},
  primaryclass  = {cs},
  publisher     = {arXiv},
  doi           = {10.48550/arXiv.2303.14737},
  urldate       = {2023-12-06},
  archiveprefix = {arXiv}
}

@misc{2021Akinola_Dynamic_Grasp_with_Reachability,
  title         = {Dynamic {{Grasping}} with {{Reachability}} and {{Motion Awareness}}},
  author        = {Akinola, Iretiayo and Xu, Jingxi and Song, Shuran and Allen, Peter K.},
  year          = {2021},
  month         = mar,
  number        = {arXiv:2103.10562},
  eprint        = {2103.10562},
  primaryclass  = {cs},
  publisher     = {arXiv},
  doi           = {10.48550/arXiv.2103.10562},
  urldate       = {2025-01-21},
  archiveprefix = {arXiv}
}

@article{2020Jenelten_online_foothold_optimization,
  title    = {Perceptive {{Locomotion}} in {{Rough Terrain}} -- {{Online Foothold Optimization}}},
  author   = {Jenelten, Fabian and Miki, Takahiro and Vijayan, Aravind E and Bjelonic, Marko and Hutter, Marco},
  year     = {2020},
  month    = oct,
  journal  = {IEEE Robotics and Automation Letters},
  volume   = {5},
  number   = {4},
  pages    = {5370--5376},
  issn     = {2377-3766},
  doi      = {10.1109/LRA.2020.3007427},
  urldate  = {2025-01-20}
}

@misc{2019Villarreal_Foothold_Adaptation_CNN,
  title         = {Fast and {{Continuous Foothold Adaptation}} for {{Dynamic Locomotion}} through {{CNNs}}},
  author        = {Villarreal, Octavio and Barasuol, Victor and Camurri, Marco and Franceschi, Luca and Focchi, Michele and Pontil, Massimiliano and Caldwell, Darwin G. and Semini, Claudio},
  year          = {2019},
  month         = feb,
  number        = {arXiv:1809.09759},
  eprint        = {1809.09759},
  publisher     = {arXiv},
  doi           = {10.48550/arXiv.1809.09759},
  urldate       = {2024-11-22},
  archiveprefix = {arXiv}
}

@article{2018Tonneau_Acyclic_Contact_Planner,
  title     = {An {{Efficient Acyclic Contact Planner}} for {{Multiped Robots}}},
  author    = {Tonneau, Steve and Del Prete, Andrea and Pettre, Julien and Park, Chonhyon and Manocha, Dinesh and Mansard, Nicolas},
  year      = {2018},
  month     = jun,
  journal   = {IEEE Transactions on Robotics},
  volume    = {34},
  number    = {3},
  pages     = {586--601},
  issn      = {1552-3098, 1941-0468},
  doi       = {10.1109/TRO.2018.2819658},
  urldate   = {2024-06-28},
  copyright = {https://ieeexplore.ieee.org/Xplorehelp/downloads/license-information/IEEE.html},
  langid    = {english}
}

@inproceedings{2007Zacharias_Workspace_reachability,
  title      = {Capturing Robot Workspace Structure: Representing Robot Capabilities},
  shorttitle = {Capturing Robot Workspace Structure},
  booktitle  = {2007 {{IEEE}}/{{RSJ International Conference}} on {{Intelligent Robots}} and {{Systems}}},
  author     = {Zacharias, Franziska and Borst, Christoph and Hirzinger, Gerd},
  year       = {2007},
  month      = oct,
  pages      = {3229--3236},
  issn       = {2153-0866},
  doi        = {10.1109/IROS.2007.4399105},
  urldate    = {2025-01-20}
}

@article{2013Papadakis_Survey_Traversability,
  title      = {Terrain Traversability Analysis Methods for Unmanned Ground Vehicles: {{A}} Survey},
  shorttitle = {Terrain Traversability Analysis Methods for Unmanned Ground Vehicles},
  author     = {Papadakis, Panagiotis},
  year       = {2013},
  month      = apr,
  journal    = {Engineering Applications of Artificial Intelligence},
  volume     = {26},
  number     = {4},
  pages      = {1373--1385},
  issn       = {0952-1976},
  doi        = {10.1016/j.engappai.2013.01.006},
  urldate    = {2025-01-20}
}

@article{2022Borges_Survey_on_Traversability,
  title      = {A {{Survey}} on {{Terrain Traversability Analysis}} for {{Autonomous Ground Vehicles}}: {{Methods}}, {{Sensors}}, and {{Challenges}}},
  shorttitle = {A {{Survey}} on {{Terrain Traversability Analysis}} for {{Autonomous Ground Vehicles}}},
  author     = {Borges, Paulo and Peynot, Thierry and Liang, Sisi and Arain, Bilal and Wildie, Matthew and Minareci, Melih and Lichman, Serge and Samvedi, Garima and Sa, Inkyu and Hudson, Nicolas and Milford, Michael and Moghadam, Peyman and Corke, Peter},
  year       = {2022},
  month      = mar,
  journal    = {Field Robotics},
  volume     = {2},
  number     = {1},
  pages      = {1567--1627},
  issn       = {27713989},
  doi        = {10.55417/fr.2022049},
  urldate    = {2025-01-20},
  langid     = {english}
}

@article{2021Buchanan,
  title      = {Perceptive Whole-Body Planning for Multilegged Robots in Confined Spaces},
  shorttitle = {{{ANYmal}} Confined Space},
  author     = {Buchanan, Russell and Wellhausen, Lorenz and Bjelonic, Marko and Bandyopadhyay, Tirthankar and Kottege, Navinda and Hutter, Marco},
  year       = {2021},
  journal    = {Journal of Field Robotics},
  volume     = {38},
  number     = {1},
  pages      = {68--84},
  issn       = {1556-4967},
  doi        = {10.1002/rob.21974},
  urldate    = {2023-10-07},
  copyright  = {{\copyright} 2020 The Authors. Journal of Field Robotics published by Wiley Periodicals LLC},
  langid     = {english}
}

@misc{2024Miki_walk_in_confined_space,
  title         = {Learning to Walk in Confined Spaces Using {{3D}} Representation},
  author        = {Miki, Takahiro and Lee, Joonho and Wellhausen, Lorenz and Hutter, Marco},
  year          = {2024},
  month         = feb,
  number        = {arXiv:2403.00187},
  eprint        = {2403.00187},
  publisher     = {arXiv},
  doi           = {10.48550/arXiv.2403.00187},
  urldate       = {2024-11-19},
  archiveprefix = {arXiv}
}

@article{2022Xu_MCTS,
  title      = {Contact {{Sequence Planning}} for {{Hexapod Robots}} in {{Sparse Foothold Environment Based}} on {{Monte-Carlo Tree}}},
  shorttitle = {{{MCTS}}},
  author     = {Xu, Peng and Ding, Liang and Wang, Zhikai and Gao, Haibo and Zhou, Ruyi and Gong, Zhaopei and Liu, Guangjun},
  year       = {2022},
  month      = apr,
  journal    = {IEEE Robotics and Automation Letters},
  volume     = {7},
  number     = {2},
  pages      = {826--833},
  issn       = {2377-3766},
  doi        = {10.1109/LRA.2021.3133610},
  urldate    = {2024-08-26}
}

@article{1996Barber_QuickHull,
  title    = {The Quickhull Algorithm for Convex Hulls},
  author   = {Barber, C. Bradford and Dobkin, David P. and Huhdanpaa, Hannu},
  year     = {1996},
  month    = dec,
  journal  = {ACM Trans. Math. Softw.},
  volume   = {22},
  number   = {4},
  pages    = {469--483},
  issn     = {0098-3500},
  doi      = {10.1145/235815.235821},
  urldate  = {2024-08-24}
}

@phdthesis{2010Diankov_IKFast,
  title      = {Automated {{Construction}} of {{Robotic Manipulation Programs}}},
  shorttitle = {{{IKFast}}},
  author     = {Diankov, Rosen},
  year       = {2010},
  month      = aug,
  langid     = {english},
  school     = {CMU}
}

@inproceedings{2000Kuffner_RRTConnect,
  title      = {{{RRT-connect}}: {{An}} Efficient Approach to Single-Query Path Planning},
  shorttitle = {{{RRT-connect}}},
  booktitle  = {Proceedings 2000 {{ICRA}}. {{Millennium Conference}}. {{IEEE International Conference}} on {{Robotics}} and {{Automation}}. {{Symposia Proceedings}} ({{Cat}}. {{No}}.{{00CH37065}})},
  author     = {Kuffner, J.J. and LaValle, S.M.},
  year       = {2000},
  month      = apr,
  volume     = {2},
  pages      = {995-1001 vol.2},
  issn       = {1050-4729},
  doi        = {10.1109/ROBOT.2000.844730},
  urldate    = {2023-10-16}
}

@article{2022Wang_GCOPTER,
  title      = {Geometrically {{Constrained Trajectory Optimization}} for {{Multicopters}}},
  shorttitle = {{{GCOPTER}}},
  author     = {Wang, Zhepei and Zhou, Xin and Xu, Chao and Gao, Fei},
  year       = {2022},
  month      = oct,
  journal    = {IEEE Transactions on Robotics},
  volume     = {38},
  number     = {5},
  pages      = {3259--3278},
  issn       = {1941-0468},
  doi        = {10.1109/TRO.2022.3160022},
  urldate    = {2023-10-11}
}

@incollection{2016Fankhauser_GridMap,
  title      = {A {{Universal Grid Map Library}}: {{Implementation}} and {{Use Case}} for {{Rough Terrain Navigation}}},
  shorttitle = {Grid {{Map}}},
  booktitle  = {In: {{Robot Operating System}} ({{ROS}})},
  author     = {Fankhauser, P{\'e}ter and Hutter, Marco},
  year       = {2016},
  month      = jan,
  volume     = {625},
  doi        = {10.1007/978-3-319-26054-9_5},
  isbn       = {978-3-319-26052-5}
}

@article{2014Schulman_TrajOpt,
  title      = {Motion Planning with Sequential Convex Optimization and Convex Collision Checking},
  shorttitle = {{{TrajOpt}}},
  author     = {Schulman, John and Duan, Yan and Ho, Jonathan and Lee, Alex and Awwal, Ibrahim and Bradlow, Henry and Pan, Jia and Patil, Sachin and Goldberg, Ken and Abbeel, Pieter},
  year       = {2014},
  month      = aug,
  journal    = {The International Journal of Robotics Research},
  volume     = {33},
  number     = {9},
  pages      = {1251--1270},
  issn       = {0278-3649, 1741-3176},
  doi        = {10.1177/0278364914528132},
  urldate    = {2023-10-16}
}

@article{1998LaValle_RRT,
  title      = {Rapidly-Exploring Random Trees : A New Tool for Path Planning},
  shorttitle = {{{RRT}}},
  author     = {LaValle, S.},
  year       = {1998},
  journal    = {The annual research report},
  urldate    = {2023-10-16}
}

@article{1996Kavraki_PRM,
  title      = {Probabilistic Roadmaps for Path Planning in High-Dimensional Configuration Spaces},
  shorttitle = {{{PRM}}},
  author     = {Kavraki, L.E. and Svestka, P. and Latombe, J.-C. and Overmars, M.H.},
  year       = {Aug./1996},
  journal    = {IEEE Transactions on Robotics and Automation},
  volume     = {12},
  number     = {4},
  pages      = {566--580},
  issn       = {1042296X},
  doi        = {10.1109/70.508439},
  urldate    = {2023-10-16}
}

@article{2019Buchanan_Weaver,
  title      = {Walking {{Posture Adaptation}} for {{Legged Robot Navigation}} in {{Confined Spaces}}},
  shorttitle = {Weaver},
  author     = {Buchanan, Russell and Bandyopadhyay, Tirthankar and Bjelonic, Marko and Wellhausen, Lorenz and Hutter, Marco and Kottege, Navinda},
  year       = {2019},
  month      = apr,
  journal    = {IEEE Robotics and Automation Letters},
  volume     = {4},
  number     = {2},
  pages      = {2148--2155},
  issn       = {2377-3766},
  doi        = {10.1109/LRA.2019.2899664},
  urldate    = {2023-09-25}
}

@article{2013zucker_CHOMP,
  title      = {{{CHOMP}}: {{Covariant Hamiltonian}} Optimization for Motion Planning},
  shorttitle = {{{CHOMP}}},
  author     = {Zucker, Matt and Ratliff, Nathan and Dragan, Anca D. and Pivtoraiko, Mihail and Klingensmith, Matthew and Dellin, Christopher M. and Bagnell, J. Andrew and Srinivasa, Siddhartha S.},
  year       = {2013},
  month      = aug,
  journal    = {International Journal of Robotics Research},
  volume     = {32},
  number     = {9-10},
  pages      = {1164--1193},
  issn       = {0278-3649},
  doi        = {10.1177/0278364913488805},
  urldate    = {2023-10-07}
}

@article{2023fahmi_ViTAL,
  title      = {{{ViTAL}}: {{Vision-Based Terrain-Aware Locomotion}} for {{Legged Robots}}},
  shorttitle = {{{ViTAL}}},
  author     = {Fahmi, Shamel and Barasuol, Victor and Esteban, Domingo and Villarreal, Octavio and Semini, Claudio},
  year       = {2023},
  month      = apr,
  journal    = {IEEE Transactions on Robotics},
  volume     = {39},
  number     = {2},
  pages      = {885--904},
  issn       = {1941-0468},
  doi        = {10.1109/TRO.2022.3222958},
  urldate    = {2023-09-25}
}

@misc{2022grandia_pnmpc,
  title        = {Perceptive {{Locomotion}} through {{Nonlinear Model Predictive Control}}},
  shorttitle   = {Perceptive {{NMPC}}},
  author       = {Grandia, Ruben and Jenelten, Fabian and Yang, Shaohui and Farshidian, Farbod and Hutter, Marco},
  year         = {2022},
  month        = aug,
  number       = {arXiv:2208.08373},
  eprint       = {2208.08373},
  primaryclass = {cs},
  publisher    = {arXiv},
  doi          = {10.48550/arXiv.2208.08373},
  urldate      = {2023-08-02}
}

@inproceedings{2018fankhauser,
  title      = {Robust {{Rough-Terrain Locomotion}} with a {{Quadrupedal Robot}}},
  shorttitle = {{{ANYmal Terrain-Loco}}},
  booktitle  = {2018 {{IEEE International Conference}} on {{Robotics}} and {{Automation}} ({{ICRA}})},
  author     = {Fankhauser, Peter and Bjelonic, Marko and Dario Bellicoso, C. and Miki, Takahiro and Hutter, Marco},
  year       = {2018},
  month      = may,
  pages      = {5761--5768},
  publisher  = {IEEE},
  address    = {Brisbane, QLD},
  doi        = {10.1109/ICRA.2018.8460731},
  urldate    = {2023-09-27}
}
